\documentclass[11pt]{article}
\usepackage{acl}
\usepackage{times}
\usepackage{latexsym}
\usepackage[T1]{fontenc}
\usepackage[utf8]{inputenc}
\usepackage{microtype}
\usepackage{graphicx}
\usepackage{booktabs}
\usepackage{amsmath}
\usepackage{amssymb}
\usepackage{amsthm}
\usepackage{xspace}

\newtheorem{proposition}{Proposition}
\newtheorem{corollary}[proposition]{Corollary}

\newcommand{\base}{InertiaKV\xspace}
\newcommand{\method}{InertiaKV-Lazy\xspace}
\newcommand{\lazyfour}{InertiaKV-Lazy4\xspace}
\newcommand{\revision}[1]{#1}

\title{What Matters for Aggressive Decoding-Time KV Eviction?\\ Temporal Aggregation and Ranking Preservation}

\author{
\textbf{Bo Zeng}$^{1}$ \quad
\textbf{Yu Zhao}$^{2}$ \quad
\textbf{Yefeng Liu}$^{1}$ \quad
\textbf{Zhihong Lu}$^{1}$ \\
\textbf{Xuanfan Ni}$^{2}$ \quad
\textbf{Xintong Wang}$^{2,3}$\thanks{Corresponding author xintong.wang@uni-hamburg.de.} \\
$^{1}$Ant Group \quad
$^{2}$Alibaba Group \quad
$^{3}$Universit\"at Hamburg \\
\texttt{liheng.zb@antgroup.com, hanfeng.wxt@alibaba-inc.com} \\
\url{https://github.com/BobTsang-NLP/InertiaKV}
}

\begin{document}
\maketitle

\begin{abstract}
Decoding-time KV cache compression research focuses heavily on designing better token scoring functions, while the temporal rule that aggregates scores across decode steps is often treated as an implementation detail. Under aggressive KV compression, we find that exponential-moving-average (EMA) aggregation makes approximately order-preserving scorer modifications largely indistinguishable at the eviction-set level. Value-norm and entropy variants remain highly correlated with attention and produce nearly unchanged retention sets, whereas KeyDiff, key norm, recency, and a learned scorer alter the ranking and degrade substantially. We associate this stability with the evaluated aggregation, which couples layer weighting and temporal retention. Building on this observation, we introduce \textit{\textbf{InertiaKV}}, an EMA-based decoding-time eviction method, and \textit{\textbf{InertiaKV-Lazy}}, its periodic-refresh variant, which yields 1.34--1.46$\times$ decode throughput relative to full refresh InertiaKV. We also study Score-Free decoding as a separate empirical operating point: it scores the full context once at the first decode step, freezes that ranking, and incurs an average quality change of $+0.03$ while removing all subsequent scoring. Across six open-weight backbones and the LongBench, LongBench-v2, and RULER benchmarks, the results identify temporal aggregation and ranking preservation as distinct, consequential design factors; they do not imply that scoring quality is irrelevant in general.
\end{abstract}

\section{Introduction}

Decoding-time KV cache compression \citep{h2o2023,liu2023scissorhands,xiao2024tova} is a sequential decision problem: at each generation step, the compressor must decide which tokens to retain and which to permanently evict under a fixed memory budget. This decision decomposes into two \revision{commonly discussed} design axes: (1) a \emph{per-step scoring function} that estimates each token's current importance, and (2) an \emph{aggregation rule} that combines scores across \revision{decoder layers and decode steps} into the running utility estimate that drives eviction.

Existing methods focus primarily on the first axis: H2O \citep{h2o2023} accumulates attention mass, TOVA \citep{xiao2024tova} ranks by last-query attention, ScissorHands \citep{liu2023scissorhands} exploits attention persistence patterns, AdaKV \citep{li2024adakv} adapts budgets across heads, FastGen \citep{ge2024model} uses model-aware adaptive compression, KeyDiff \citep{liu2025keydiff} scores by key geometry. Prefill-time methods such as SnapKV \citep{xia2024snapkv} select tokens before generation begins; orthogonal techniques reduce per-token cost via quantization \citep{liu2024kivi} or cross-layer sharing \citep{brandon2024cla}. The \revision{aggregation} axis (single-step, cumulative sum, or bounded-memory smoothing) is treated as an implementation detail rather than a first-class design choice.

At moderate compression rates, existing methods generally maintain acceptable quality. Under aggressive compression, however, this no longer holds: on RULER \citep{hsieh2024ruler}, as the KV budget shrinks from 50\% to 10\%, TOVA, SnapKV, and AdaKV all degrade substantially, albeit at different rates, despite employing fundamentally different scoring functions (Table~\ref{tab:intro-degradation}). This shared degradation across diverse scorers suggests that scoring alone may not determine quality in this regime. Prior work has rarely isolated scoring from temporal aggregation; they are typically evaluated as a coupled system, leaving open which axis governs quality at such aggressive budgets.

\begin{table}[t]
\centering
\scriptsize
\begin{tabular*}{\columnwidth}{@{\extracolsep{\fill}}lrrrr}
\toprule
 & Full & 50\% & 25\% & 10\% \\
\midrule
TOVA & 94.59 & 77.10 & 65.83 & 48.34 \\
SnapKV & 94.59 & 68.79 & 43.26 & 24.33 \\
AdaKV+EA & 94.59 & 89.24 & 71.16 & 28.06 \\
\bottomrule
\end{tabular*}
\caption{\textbf{RULER score on Llama-3.1-8B as KV budget decreases.} Methods with different scoring functions exhibit different degradation rates as the budget shrinks.}
\label{tab:intro-degradation}
\end{table}

This motivates a shift in focus toward the neglected second axis: \revision{score aggregation}. We adopt exponential moving average (EMA) smoothing, which accumulates scores during decoding with bounded memory and a tunable coefficient $\alpha$. \revision{In the evaluated implementation, the running state is updated sequentially as scores are obtained from the decoder layers and is retained across decode steps, thereby coupling layer weighting with temporal retention.} To \revision{isolate scorer variation under this fixed aggregation procedure}, we conduct an initial experiment on Llama-3.1-8B at 90\% compression: we fix the \revision{sequential layer--temporal aggregation rule} and systematically vary the scoring function. Replacing the attention-based scorer with value-norm or entropy weighting produces near-identical eviction sets (Figure~\ref{fig:scoring-irrelevance}; mean pairwise Jaccard $= 0.97$). After the \revision{running} ranking stabilizes, approximately order-preserving scorer modifications rarely change keep/evict membership.

To assess the generality of this finding, we extend the experiment to Qwen2.5-7B and Llama-3.3-70B on LongBench \citep{bai2023longbench}, LongBench-v2 \citep{bai2024longbenchv2}, and RULER. The pattern is consistent: perturbations that approximately preserve the utility ranking \revision{produce similar retention sets under the evaluated aggregation}, while scorers that substantially reorder it degrade quality. Under these aggressive-compression settings, \revision{the evaluated sequential layer--temporal aggregation} makes approximately order-preserving scorer modifications largely indistinguishable at the eviction-set level; this conclusion does not extend to scorers that substantially alter the utility ranking.

\revision{We associate this with \emph{ranking inertia} under the evaluated sequential layer--temporal aggregation: repeated updates are associated with more stable retention-set membership (\S\ref{sec:momentum-dominance}), but the present experiments do not separate cross-step memory from the recurrence's implicit later-layer weighting.} This stability comes at a cost: stronger smoothing delays adaptation to relevance shifts, producing an \revision{observed} \emph{boundary-error vs.\ lag-error} tradeoff analyzed in \S\ref{sec:method-aggregation}. This persistence motivates InertiaKV-Lazy, which periodically refreshes scores and yields $1.34$--$1.46\times$ decode throughput relative to full-refresh InertiaKV. As a separate operating point, Score-Free computes scores across the decoder layers at the first decode step and freezes the resulting ranking; in our evaluation it eliminates all subsequent score updates with negligible average quality change ($\Delta{=}{+}0.03$ points). In summary, this paper makes three contributions:
\begin{itemize}
    % \item We identify ranking inertia under evaluated sequential layer-temporal aggregation across 7B-70B models on LongBench, LongBench-v2, and RULER: approximately order-preserving scorer changes leave the eviction set nearly unchanged, whereas ranking altering scorers can degrade quality substantially.
    \item We identify ranking inertia in sequential layer-temporal aggregation across 7B--70B models on LongBench, LongBench-v2, and RULER: approximately order-preserving scorer changes leave the eviction set nearly unchanged, whereas ranking-altering scorers can sharply degrade quality.
    \item We characterize the observed boundary-error vs.\ lag-error tradeoff and provide a rank-stability bound, explaining why stronger temporal smoothing can stabilize eviction decisions yet delay adaptation to relevance shifts. Our negative result further shows that local noise and drift alone cannot reliably guide adaptive smoothing.
    \item We translate decode-time ranking persistence into operating points: InertiaKV-Lazy yields $1.34$--$1.46\times$ decode speedup over full-refresh InertiaKV, while Score-Free eliminates all subsequent score updates after first-step initialization with negligible average quality change. These results establish refresh frequency as explicit quality--efficiency control.
\end{itemize}

\begin{figure*}[t]
\centering
\includegraphics[width=1.0\textwidth,trim=7pt 7pt 7pt 7pt,clip]{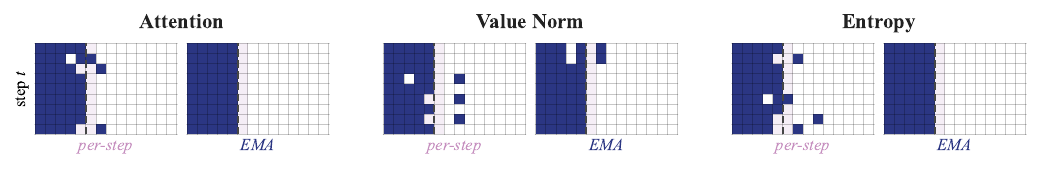}
\caption{\textbf{EMA stabilizes the retention sets of three highly rank-correlated scorers.} Each panel pairs per-step scoring (\emph{left}) with EMA-smoothed scoring (\emph{right}) for Attention, Value Norm, and Entropy. Per-step scores fluctuate, while EMA yields near-identical retention sets for these approximately order-preserving variants (mean pairwise Jaccard $= 0.97$ over top-$B$ token sets, averaged across decode steps and layers). Navy cells indicate retained tokens; columns are sorted by reference utility (high $\to$ low); the dashed line marks the budget boundary.}
\label{fig:scoring-irrelevance}
\end{figure*}

\section{Related Work}

\paragraph{Decode-time eviction methods.}
Decode-time KV eviction has evolved rapidly around the design of better scoring functions. H2O \citep{h2o2023} introduced cumulative attention mass as a heavy-hitter signal, establishing the template of score-and-evict under a fixed budget. ScissorHands \citep{liu2023scissorhands} then observed that attention patterns persist across decode steps, enabling future importance to be predicted from past attention. TOVA \citep{xiao2024tova} simplified the scoring to a single-step last-query signal, trading temporal information for lower overhead. Subsequent work enriched the scoring axis further: AdaKV \citep{li2024adakv} adapts budgets per attention head, KeyDiff \citep{liu2025keydiff} replaces attention weights with key-vector geometry, and TaDA \citep{joshi2025tada} adds mean-centering to adaptive compression. Recent work such as RocketKV \citep{behnam2025rocketkv} applies two-stage coarse-to-fine scoring for acceleration, and EvolKV \citep{yu2025evolkv} uses evolutionary search to optimize layer-wise budget allocation. Throughout this progression, the temporal rule that aggregates per-step scores into eviction decisions (cumulative sum, sliding window, or implicit single-step) has remained a background implementation choice rather than an object of study in its own right.

\paragraph{Temporal perspectives.}
Few works touch on the temporal dimension of eviction. ScissorHands' persistence hypothesis implies temporal stability but frames it as a property of the attention signal rather than of the aggregation mechanism. FAEDKV \citep{li2025faedkv} identifies recency bias in existing methods and proposes frequency-domain transforms for unbiased scoring across an infinite window---among the first efforts to treat temporal dynamics explicitly as a design problem. However, FAEDKV contributes a debiasing technique that improves the scoring signal; it does not decouple scoring from aggregation or test how temporal smoothing interacts with scorer ranking. Our finding is complementary: Under aggressive KV compression, we find that exponential-moving-average (EMA) aggregation makes approximately order-preserving scorer modifications largely indistinguishable at the eviction-set level, but does not rescue scorers whose rankings differ substantially from attention.

\paragraph{Prefill-time and sparse-attention methods.}
SnapKV \citep{xia2024snapkv} selects tokens during prefill; StreamingLLM \citep{liu2023streamingllm} retains attention sinks and a local window. These operate before or outside the decode loop. Sparse-attention systems \citep{tang2024quest,lee2024infinigen,sun2025shadowkv} maintain the full cache off-chip and load subsets per step---a fundamentally different regime from irreversible eviction.

\paragraph{Orthogonal dimensions.}
Cache size can also be reduced through quantization \citep{liu2024kivi}, cross-layer sharing \citep{brandon2024cla}, layer-wise budget shaping \citep{pyramidkv2024}, or reconstruction-based compression \citep{zhang2025kvzip}. Hardware-efficient attention implementations \citep{dao2023flashattention2} reduce per-step cost but do not address cache growth. These are composable with eviction-based methods but orthogonal to the question studied here. To our knowledge, prior work rarely isolates scoring from temporal aggregation as independently controllable design axes.

\section{Method}

We study decode-time KV cache compression under aggressive budgets. Given a prefill cache of length $L_0$ and compression ratio $\rho$, the decode-side budget is $B = \lfloor (1{-}\rho)L_0 \rfloor$. Unlike sparse attention, eviction is irreversible: removed KV pairs cannot be recovered.

\subsection{Temporal Utility Aggregation}
\label{sec:method-aggregation}

At each decode step $t$, we extract the last-query attention weights \citep{vaswani2017attention} from all $H$ heads and compute a per-token importance score by mean-pooling across heads:
\begin{equation}
    s_t(i) = \frac{1}{H}\sum_{h=1}^{H} a_t^{(h)}(i), \quad i = 1,\ldots,L_t,
\label{eq:score}
\end{equation}
where $a_t^{(h)}(i)$ is head $h$'s attention to cached position $i$. We maintain a utility vector $m_t \in \mathbb{R}^{L_t}$, where $m_t(i)$ estimates the aggregated importance of cached token $i$ through step $t$. Whenever the cache exceeds $B$, we evict the tokens with the lowest $m_t(i)$. The aggregation rule is the primary design choice:
\begin{align}
    \makebox[5.5em][l]{\text{Single-step:}} m_t &= s_t, \label{eq:single-step} \\
    \makebox[5.5em][l]{\text{Cumulative:}}  m_t &= m_{t-1} + s_t, \label{eq:cumulative} \\
    \makebox[5.5em][l]{\text{EMA:}}         m_t &= \alpha\, m_{t-1} + (1{-}\alpha)\,s_t. \label{eq:ema}
\end{align}
These three rules span the stability--adaptivity axis: single-step is maximally adaptive but noisy; cumulative never forgets but adapts slowly to relevance shifts; EMA ($\alpha \in (0,1)$) interpolates between them with bounded memory. Equations~\ref{eq:single-step}--\ref{eq:ema} present aggregation at the decode-step level. In our implementation, compression is applied at multiple transformer layers, and the EMA state is updated after each such layer; consequently, $\alpha$ controls both persistence across decode steps and the relative influence of successive layers.

Under aggressive compression, the choice of $\alpha$ exposes a fundamental tradeoff between two error modes. \emph{Boundary errors}: noisy score fluctuations near the budget boundary permanently evict useful tokens. \emph{Lag errors}: an overly slow estimator retains stale tokens after relevance has shifted. We express this as a two-term risk:
\begin{equation}
    \mathcal{L}(\alpha)
    =
    \underbrace{\mathcal{E}_{\mathrm{boundary}}(\alpha)}_{\text{ranking instability}}
    +
    \underbrace{\mathcal{E}_{\mathrm{lag}}(\alpha)}_{\text{adaptation delay}} .
\label{eq:risk-decomposition}
\end{equation}
We treat this decomposition as a qualitative design lens rather than a formal objective; the quantitative bound on ranking stability is given in Proposition~\ref{prop:rank-stability}. This framing guides the alpha-sensitivity analysis (\S\ref{sec:temporal-sensitivity}) and explains why adaptive scheduling fails (Appendix~\ref{sec:adaptive-alpha}).

We instantiate this framework as \textbf{\base}: decode-side EMA aggregation (Eq.~\ref{eq:ema}) with $\alpha{=}0.8$ and a proportional budget $B{=}\lfloor(1{-}\rho)L_0\rfloor$. The name reflects the core mechanism---EMA builds \emph{ranking inertia} that absorbs per-step scoring perturbations (\S\ref{sec:momentum-dominance}). Compression is applied during decoding (not prefill), making it naturally query-aware but not a prefill memory reduction. In one 79-step diagnostic trace at $\alpha{=}0.8$ and 90\% compression, the top-$B$ set remains identical to its first-step set, although within-set rankings continue to change (Appendix~\ref{sec:appendix-diagnostics}).

\subsection{Exploiting Ranking Stability}
\label{sec:method-lazy}

If utility ranking changes little between consecutive steps, the expensive query-aware score $s_t$ need not be recomputed every step. We define two operating points that progressively exploit this property.

\paragraph{Lazy refresh.}
  \method refreshes scores at the first decode step and every
  $r$ steps thereafter. Let
  $\mathcal{R}_r=\{t\geq1:(t-1)\equiv0\pmod r\}$ denote the
  set of refresh steps:
  \begin{equation}
      m_t {=}
      \begin{cases}
          \alpha m_{t-1} {+} (1{-}\alpha)s_t,
          & t \in \mathcal{R}_r, \\[2pt]
          m_{t-1}, & \text{otherwise}.
      \end{cases}
  \label{eq:lazy}
  \end{equation}
  Eviction still occurs whenever the cache exceeds $B$; only the score observation is skipped on non-refresh steps. The discrepancy from full refresh can grow with both the refresh interval and the update weight. This heuristic motivates moderate refresh intervals; the resulting quality--speed tradeoff is evaluated empirically rather than guaranteed by
  Proposition~\ref{prop:rank-stability}.

\paragraph{Score-Free decoding.}\label{sec:method-scorefree}
Score-Free is a separate empirical operating point rather than the $r\!\to\!\infty$ consequence of Proposition~\ref{prop:rank-stability}. EMA is inactive during prefill and the momentum state is initialized as $m_0=0$. At the first decode step, Score-Free sequentially processes the eligible layer scores and freezes the resulting state. No score is computed from step 2 onward, while eviction continues using this fixed ranking. Thus Score-Free is more precisely ``score once, then never refresh.'' Its effectiveness is an empirical observation about the evaluated workloads, not a claim that the first-step ranking is universally optimal.

\begin{center}
\fbox{\begin{minipage}{0.91\columnwidth}
\textbf{Score-Free decoding}\quad
\textbf{Input:} prefill KV cache $\mathcal{C}_0$, budget $B$, EMA coefficient $\alpha$.\\
\textbf{Initialize:} $m_0(i)\leftarrow 0$ for every prefill position $i$.\\
\textbf{First decode step:} sequentially process the eligible full-context layer scores; freeze the resulting state and evict the lowest-ranked entries if $|\mathcal{C}_1|>B$.\\
\textbf{Later steps $t\geq2$:} do not compute $s_t$ and keep the prefill-token ranking fixed by $m_1$; initialize each new token to the current mean momentum, and evict whenever the cache exceeds $B$.
\end{minipage}}
\end{center}

\noindent Controlled auxiliary ablations (static reserve, first-eviction rescue) are reported in Appendix~\ref{sec:appendix-controlled}.

\section{Experimental Setup}
\subsection{Benchmarks}
Our main evaluation uses LongBench \citep{bai2023longbench}, LongBench-v2 \citep{bai2024longbenchv2}, and RULER \citep{hsieh2024ruler}. LongBench covers extractive QA and generation tasks. RULER comprises 13 synthetic subtasks spanning single-key, multi-key, multi-query, and multi-value retrieval, plus variable tracking and word extraction; multi-evidence subtasks (multikey, multiquery) are the hardest for eviction methods because they require retaining multiple dispersed tokens simultaneously---the regime where aggressive compression most frequently fails. We additionally include a needle-in-a-haystack grid as an auxiliary retrieval stress test, a MoE stress check on Qwen3-30B-A3B, and architecture/scale stress tests in the appendix. LongBench and RULER are on a 0--100 scale; LongBench-v2 uses proportion correct (0--1) following its official protocol.

\subsection{Models}
We evaluate on Llama-3.1-8B-Instruct \citep{grattafiori2024llama3} and Qwen2.5-7B-Instruct \citep{yang2024qwen2} for the main results. These two models represent widely used open-weight families and use different grouped-query attention (GQA; \citealp{ainslie2023gqa}) configurations: Llama uses 8 KV heads with $4\times$ query-head grouping, while Qwen uses 4 KV heads with $7\times$ grouping. This provides coverage of two distinct GQA configurations. We include Llama-3.3-70B-Instruct to test whether the same empirical pattern extends beyond 7--8B parameters, and Qwen3-30B-A3B-Instruct as a minimal MoE stress check where sparse expert activation may interact differently with cache eviction. Appendix stress tests use Mistral-7B-Instruct-v0.3 and Qwen2.5-14B-Instruct to probe transfer limits rather than to support the main claim. All experiments run on NVIDIA H100 and H200 GPUs.

\subsection{Compared Methods}
Baselines span the landscape of KV cache compression: TOVA \citep{xiao2024tova} (single-step last-query attention), SnapKV \citep{xia2024snapkv} (prefill-time observation-window selection), AdaKV + Expected Attention \citep{li2024adakv} (adaptive per-head budgets), KeyDiff \citep{liu2025keydiff} (key-vector geometry), and KVzip \citep{zhang2025kvzip} (reconstruction-based compression). We report both the full-refresh \base scorer and practical \method variants. All comparisons use 90\% compression---the aggressive regime in which temporal effects are most pronounced in our experiments (Table~\ref{tab:intro-degradation}). We also include a cumulative-attention baseline that replaces EMA (Eq.~\ref{eq:ema}) with the non-decaying accumulator (Eq.~\ref{eq:cumulative}).

\section{Results}

\subsection{Order-Preserving Scoring Perturbations Are Absorbed by EMA}
\label{sec:scoring-ablation}

Among the three aggregation rules defined in \S\ref{sec:method-aggregation}, EMA ($\alpha{=}0.8$) substantially outperforms single-step on retrieval-heavy settings ($+1.8$ on LongBench and $+21.6$ on RULER on Llama), while cumulative aggregation is competitive but mixed across models and benchmarks (Appendix~\ref{sec:appendix-temporal}). Given this strong baseline, we now ask: are order-preserving modifications to the scoring function absorbed by EMA momentum? We test two alternative scorers on Qwen2.5-7B across all eight LongBench tasks at 90\% compression: \textbf{VNorm}, which multiplies each token's attention weight by the mean-normalized $\ell_2$ norm of its value state; and \textbf{Entropy weighting}, a per-layer scalar that modulates the step's overall influence without changing the relative ordering among tokens. Both modifications are applied per step before EMA aggregation.

\begin{figure*}[t]
\centering
\includegraphics[width=0.8\textwidth,trim=7pt 9pt 7pt 7pt,clip]{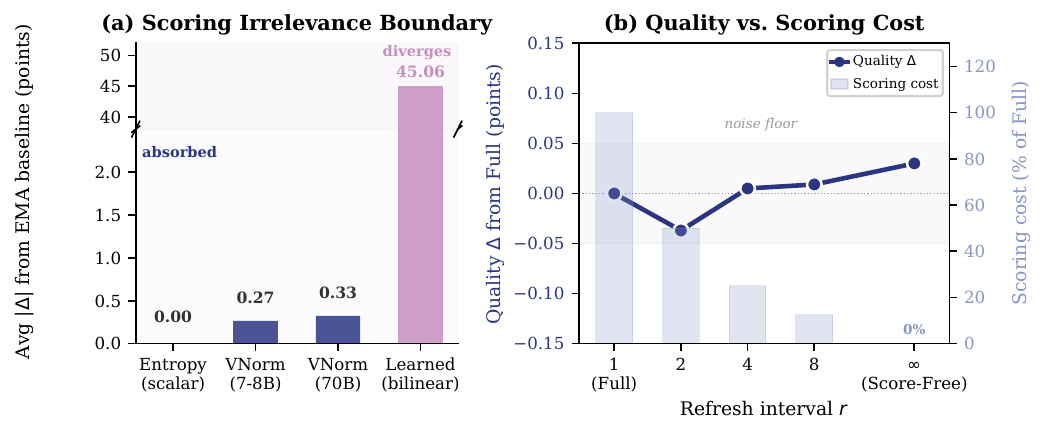}
\caption{\textbf{Core findings.} \textbf{(a)} Approximately order-preserving perturbations (Entropy, VNorm) are absorbed by EMA momentum (avg $|\Delta| < 0.35$), while the evaluated learned scorer changes the ranking and diverges ($|\Delta| = 45.06$). \textbf{(b)} Lazy refresh reduces scoring frequency with little quality change. Score-Free is shown as a separate endpoint that scores once at the first decode step and performs no subsequent refresh.}
\label{fig:main-result}
\end{figure*}

Figure~\ref{fig:main-result}(a) shows the result. Entropy weighting produces \emph{exactly} zero task-score deviation in this evaluation. Its per-layer scalar preserves the within-layer token ordering, although different layer weights can in principle alter the final aggregate ranking. VNorm, a genuinely per-token perturbation, produces small deviations (avg $|\Delta| \leq 0.33$), consistent with EMA momentum absorbing this token-wise reweighting. In contrast, the $\alpha{=}0$ endpoint drops the average by 1.33 points, and random eviction collapses to 12.45. Here $\alpha{=}0$ retains only the last processed layer, so the ablation is not purely temporal. The same pattern holds on Llama-3.1-8B (VNorm $\leq 0.35$; $\alpha{=}0$ drops from 43.88 to 42.11).

\paragraph{Scale check: Llama-3.3-70B.}
To test whether this absorption pattern extends beyond the 7--8B scale, we repeat the ablation on Llama-3.3-70B-Instruct at 90\% compression (Table~\ref{tab:scoring-ablation-70b}). At this scale, the per-step perturbation produces small residual differences ($\leq 0.67$ points), while the $\alpha{=}0$ endpoint drops by 4--13 points across all tasks. This contrast changes both layer weighting and temporal retention, rather than momentum alone.

\begin{table}[t]
\centering
\scriptsize
\begin{tabular*}{\columnwidth}{@{\extracolsep{\fill}}lrrrr}
\toprule
 & No Press & EMA & VNorm & $\alpha\!=\!0$ \\
\midrule
hotpotqa & 62.86 & 63.45 & 63.44 & 59.26 \\
2wikimqa & 69.22 & 69.41 & 68.96 & 63.69 \\
qasper & 49.39 & 45.53 & 44.86 & 32.48 \\
musique & 48.88 & 48.55 & 48.35 & 42.05 \\
\midrule
\textbf{Avg} & \textbf{57.59} & \textbf{56.74} & \textbf{56.40} & 49.37 \\
\bottomrule
\end{tabular*}
\caption{\textbf{Scoring-function ablation on Llama-3.3-70B-Instruct at 90\% compression.} VNorm deviates from EMA by at most 0.67 points (avg $\Delta{=}{-}0.33$), while $\alpha{=}0$ drops by 7.37 points on average.}

%This contrast is not purely temporal.}
\label{tab:scoring-ablation-70b}
\end{table}

\paragraph{Ranking correlation defines the observed boundary.}
We measure each scorer against attention at layer 15 using Spearman correlation and top-$B$ Jaccard at a 10\% budget. VNorm and entropy remain highly correlated ($\rho\!\approx\!0.97$) and preserve quality. KeyDiff changes the ordering ($\rho{=}0.49$, Jaccard $=0.38$) and loses roughly 9--14 LongBench points; key norm ($\rho{=}0.00$) and recency ($\rho{=}0.12$) collapse on retrieval. The learned scorer is a 327,680-parameter bilinear head, $\text{score}(q,k)=(W_k k)^\top(W_q q)/\sqrt{64}$, trained by self-distillation with MSE$+$KL ($\lambda{=}0.1$, $\tau{=}1.0$). Its held-out ranking is far from attention ($\rho\!\approx\!0.05$, Jaccard $\approx0.08$); it loses 1.94 points on LongBench and 45.06 on RULER. These results establish an empirical correlation spectrum, not a proof that every non-attention scorer must fail: our evidence does not rule out a competitive alternative scorer with a different signal.

\subsection{Score Refresh Can Be Reduced or Eliminated}
\label{sec:score-free}

Because the decode-time ranking changes slowly after EMA stabilization, frequent score updates may be redundant though scorer quality still matters. We test this by reducing refresh frequency.

\paragraph{Lazy refresh.} \lazyfour is the default operating point: benchmark-level quality deltas versus full refresh are small in magnitude. LongBench confidence intervals include zero on both backbones; RULER intervals also include zero except for a small Llama regression ($\Delta{=}{-}0.151$, 95\% CI $[-0.223,-0.078]$). For Lazy4, benchmark deltas are below 0.05 except for the Llama RULER regression reported above; $r{=}4$ provides a practical quality--speed compromise. Lazy4 skips scoring on three out of every four steps, reducing layer-scoring operations by $4\times$ and yielding 1.34$\times$ and 1.46$\times$ decode speedups on Llama and Qwen respectively at 64k context (Figure~\ref{fig:lazy-tradeoff} in Appendix~\ref{sec:appendix-details}).

\paragraph{Score-Free decoding.} As specified in \S\ref{sec:method-scorefree}, Score-Free computes one full-context score at the first decode step and then freezes the ranking. Table~\ref{tab:score-free} reports Llama-3.1-8B across all eight LongBench tasks at 90\% compression. The average deviation is $+0.03$, and the aggregate paired-bootstrap 95\% confidence interval spans zero. Seven of eight tasks are statistically indistinguishable from full refresh. MultiNews is a genuine small regression ($\Delta{=}{-}0.61$, 95\% CI $[-1.01,-0.21]$, paired $t$-test $p{=}0.003$), whereas passage retrieval improves by $+0.50$ but is not significant (95\% CI $[0.00,1.50]$, $p{=}0.32$). On Qwen2.5-7B, the average deviation is $-0.01$ (max $|\Delta|{=}0.47$).

\begin{table}[t]
\centering
\scriptsize
\begin{tabular*}{\columnwidth}{@{\extracolsep{\fill}}lrrrr}
\toprule
 & Full & Lazy4 & Score-Free & $\Delta$(SF$-$Full) \\
\midrule
hotpotqa & 58.54 & 58.80 & 58.53 & $-$0.01 \\
2wikimqa & 52.52 & 52.52 & 52.62 & +0.10 \\
qasper & 42.42 & 42.48 & 42.50 & +0.08 \\
musique & 32.16 & 32.20 & 32.16 & +0.00 \\
gov\_report & 27.80 & 28.11 & 27.64 & $-$0.16 \\
multi\_news & 20.77 & 20.58 & 20.16 & $-$0.61 \\
qmsum & 23.65 & 23.85 & 24.01 & +0.36 \\
passage\_ret & 99.50 & 99.50 & 100.0 & +0.50 \\
\midrule
\textbf{Avg} & \textbf{44.67} & \textbf{44.76} & \textbf{44.70} & +0.03 \\
\bottomrule
\end{tabular*}
% \caption{\textbf{Score-Free decoding on Llama-3.1-8B across eight LongBench tasks at 90\% compression.} Score-Free computes attention once at the first decode step and never refreshes afterward. The mean deviation is $+0.03$ and its paired-bootstrap 95\% confidence interval spans zero; MultiNews is the one significant regression.}
\caption{\textbf{Score-Free decoding on Llama-3.1-8B across eight LongBench tasks at 90\% compression.} Score-Free computes attention once at the first decode step and never refreshes afterward.}
\label{tab:score-free}
\end{table}

After first-step initialization, the score-update machinery can therefore be removed from the decode loop, minimizing score-update overhead among the evaluated InertiaKV variants. Figure~\ref{fig:scoring-overhead} quantifies this cost: at 16k context on Llama-3.1-8B, score updates account for 36\% of Full Refresh decode time; Lazy4 reduces this share to 14\%; Score-Free removes subsequent score updates. Controlled ablations show that auxiliary mechanisms (static reserve, first-eviction rescue) do not improve this operating point (Appendix~\ref{sec:appendix-controlled}).

\begin{figure}[t]
\centering
\includegraphics[width=0.80\columnwidth,trim=7pt 9pt 7pt 7pt,clip]{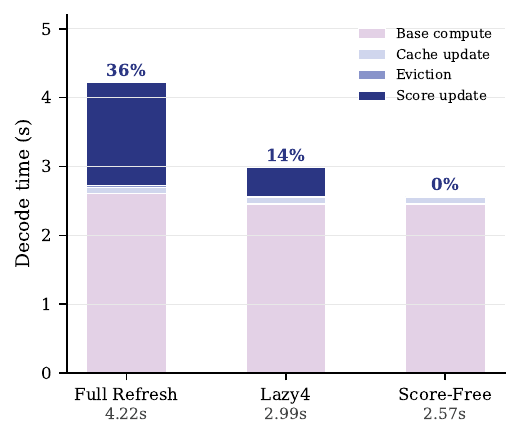}
\caption{\textbf{Decode time composition at 16k context on Llama-3.1-8B (90\% compression)}. Score updates account for 36\% of Full Refresh decode time; Lazy4 reduces this share to 14\%; Score-Free removes score updates after first-step initialization.}
\label{fig:scoring-overhead}
\end{figure}

\subsection{Why: Momentum Dominance}
\label{sec:momentum-dominance}

The scorer ablations and lazy-refresh results are consistent with a common mechanism: EMA accumulation builds \emph{ranking inertia} that can reduce the effect of large single-step perturbations. Even when per-step score perturbations are large in relative L2 norm, the current update enters EMA with weight $(1{-}\alpha){=}0.2$; an 82\% relative perturbation therefore contributes at roughly 16\% relative scale before accounting for the accumulated state. The empirical evidence for stabilization is the reduction in top-$B$ churn and high retention-set persistence reported in Appendix~\ref{sec:appendix-diagnostics}; Figure~\ref{fig:mechanism-trajectory} illustrates the mechanism. Score-Free is intentionally excluded from this explanation: it freezes the first-step ranking and is supported by the separate empirical test in \S\ref{sec:score-free}.

For analysis, we consider an idealized fixed-token process with one aggregate update per step, omitting token insertion, eviction, and the layer-sequential updates used in the reported implementation. We denote its state by $\mu_T$ to distinguish it from the dynamic cache state $m_t$. Under EMA, $\mu_T(i)=(1{-}\alpha)\sum_{\tau=1}^{T}\alpha^{T-\tau}s_\tau(i)+\alpha^T\mu_0(i)$.

\begin{proposition}[EMA Expected-Gap and Rank-Stability Bound]
\label{prop:rank-stability}
For each token $i$, assume that $\{s_t(i)\}_{t\geq 1}$ is i.i.d.\ over time with mean $\bar{s}(i)$ and variance bounded by $\sigma^2$, with independence across tokens and deterministic initialization $\mu_0(i)=0$. For two tokens $i,j$ with utility gap $\delta_{ij} \triangleq \bar{s}(i) - \bar{s}(j) > 0$:
\begin{enumerate}\setlength{\itemsep}{1pt}\setlength{\parskip}{0pt}\setlength{\parsep}{0pt}
    \item[\textbf{(a)}] The expected momentum gap converges to $\delta_{ij}$.
    \item[\textbf{(b)}] $\mathrm{Var}[\mu_T(i) - \mu_T(j)] \leq \dfrac{2(1-\alpha)}{1+\alpha}\,\sigma^2$.
    \item[\textbf{(c)}] In the stationary limit, $\displaystyle \limsup_{T\to\infty}\Pr[\mu_T(i) < \mu_T(j)] \leq \dfrac{2(1-\alpha)\sigma^2}{(1+\alpha)\,\delta_{ij}^2}$.
\end{enumerate}
\vspace{-0.5em}
\noindent The bound decreases with stronger smoothing and with a larger expected utility gap relative to score variance.
\end{proposition}

\begin{corollary}[Expected-Ranking Preservation under EMA]
\label{cor:scorer-equivalence}
If an alternative scorer applies a common strictly increasing transformation to expected utilities (i.e., $\mathbb{E}[\tilde{s}_t(i)] = f(\bar{s}(i))$), its expected EMA momentum has the same asymptotic ranking. Uniform positive scaling is exactly invisible to ranking, while the misranking bound in Proposition~\ref{prop:rank-stability} depends on the transformed gaps and variances.
\end{corollary}

\noindent Proof is in Appendix~\ref{sec:appendix-proofs}. On Llama-3.1-8B, we observe the corresponding stabilization diagnostics: raw attention scores exhibit 34\% top-$B$ Jaccard churn; EMA reduces this to 8\%; and 75\% of decode steps produce $\leq 1$ token change in the eviction set.

\begin{figure*}[t]
\centering
\includegraphics[width=0.46\textwidth,trim=125pt 0pt 125pt 175pt,clip]{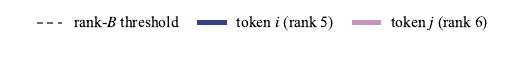}\\[-4pt]
\includegraphics[width=0.88\textwidth,trim=10pt 29pt 10pt 3pt,clip]{figures/fig_mechanism_trajectory.pdf}
% \caption{\textbf{Illustrative two-token trajectory showing how EMA can suppress boundary flips.} \textbf{(a)} Raw per-step scores for tokens near the rank-$B$ boundary cross the threshold 5 out of 16 steps. \textbf{(b)} EMA smoothing ($\alpha{=}0.8$) produces a persistent separation with zero flips in this illustration. Empirical rank-trace statistics are reported in Appendix~\ref{sec:appendix-diagnostics}.}
\caption{\textbf{Two-token illustration showing how EMA can suppress boundary flips.}
(a) Raw per-step scores of tokens near the rank-$B$ boundary cross the threshold in 5 of 16 steps.
(b) EMA smoothing ($\alpha=0.8$) yields persistent separation without flips.
Appendix~I provides empirical rank-trace statistics.}
\label{fig:mechanism-trajectory}
\end{figure*}

\begin{figure*}[t]
\centering
\includegraphics[width=0.8\textwidth,trim=7pt 7pt 7pt 7pt,clip]{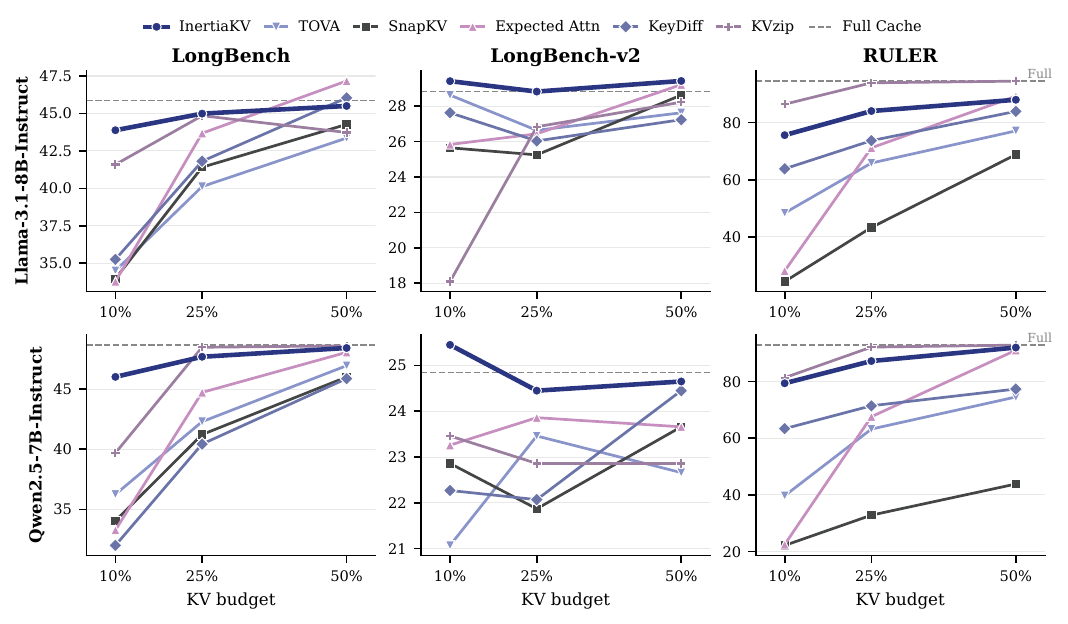}
\caption{\textbf{Quality vs.\ KV budget across three benchmarks} on Llama-3.1-8B (top) and Qwen2.5-7B (bottom) at compression ratios 50\%, 75\%, and 90\%. \base (dark blue) maintains consistently strong quality under aggressive compression, while lightweight baselines degrade sharply---especially on RULER.}
\label{fig:comparison-linechart}
\end{figure*}

\subsection{Comparison with Existing Methods}
\label{sec:reference-comparison}

Having established the interaction between EMA and scorer ranking, we now position \base against existing methods. Table~\ref{tab:reference-comparison} reports quality under the same 90\% compression protocol against TOVA, SnapKV, AdaKV + Expected Attention, KVzip, and KeyDiff.

\begin{table*}[t]
\centering
\scriptsize
\begin{tabular*}{\textwidth}{@{\extracolsep{\fill}}llrrrrrrrr}
\toprule
Model & Benchmark & Full & \base & Lazy4 & TOVA & SnapKV & AdaKV+EA & KVzip & KeyDiff \\
\midrule
Llama-3.1-8B & LongBench & 45.87 & 43.88 & 43.88 & 34.50 & 33.97 & 33.78 & 41.59 & 35.26 \\
Llama-3.1-8B & LongBench-v2 & 0.288 & 0.294 & 0.294 & 0.286 & 0.256 & 0.258 & 0.181 & 0.276 \\
Llama-3.1-8B & RULER & 94.59 & 75.60 & 75.45 & 48.34 & 24.33 & 28.06 & 86.44 & 63.82 \\
\midrule
Qwen2.5-7B & LongBench & 48.67 & 46.01 & 45.98 & 36.25 & 34.05 & 33.29 & 39.69 & 32.01 \\
Qwen2.5-7B & LongBench-v2 & 0.249 & 0.254 & 0.250 & 0.211 & 0.229 & 0.233 & 0.235 & 0.223 \\
Qwen2.5-7B & RULER & 93.00 & 79.43 & 79.44 & 39.69 & 22.14 & 22.46 & 81.40 & 63.39 \\
\bottomrule
\end{tabular*}
\caption{\textbf{Baseline comparison under the 90\% compression-ratio protocol.} LongBench-v2 scores are on a 0--1 scale; all others are on 0--100. Methods differ in compression timing, memory footprint, and computational profile.}
\label{tab:reference-comparison}
\end{table*}

\base and \lazyfour outperform all lightweight baselines (TOVA, SnapKV, AdaKV, KeyDiff) on most rows. KVzip is strongest on Llama RULER (+10.84 over \base), and we therefore report the full quality--cost comparison rather than treating the methods as directly interchangeable. In the released masking-based KVzip mode that we evaluated, KVzip obtains 86.44 RULER with 31.54\,s prefill, 10.9 token/s decode, and 64,029 retained decode-side tokens; \base obtains 75.60 with 4.86\,s prefill, 39.1 token/s decode, and 6,401 retained tokens. Thus KVzip provides higher RULER quality, while \base is 6.5$\times$ faster in prefill, 3.6$\times$ faster in decode, and retains approximately 10$\times$ fewer decode-side KV tokens. Figure~\ref{fig:kvzip-pareto} makes these tradeoffs explicit. Figure~\ref{fig:comparison-linechart} shows quality across KV budgets for the remaining baselines.

\begin{figure}[t]
\centering
\includegraphics[width=0.9\columnwidth]{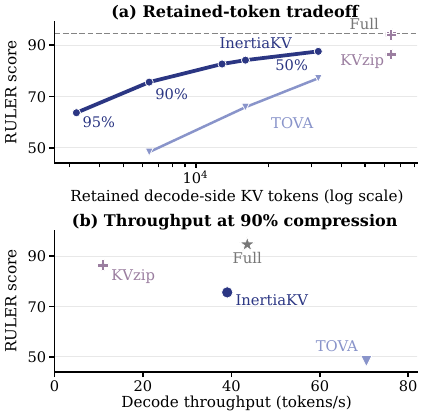}
\caption{\textbf{Quality--cost tradeoffs on Llama-3.1-8B at 64k context}: (a) RULER versus retained decode-side KV tokens across compression ratios; (b) RULER versus decode throughput at 90\% compression. KVzip denotes released masking-based mode evaluated.}
\label{fig:kvzip-pareto}
\end{figure}

\subsection{Temporal Sensitivity}
\label{sec:temporal-sensitivity}

The choice of smoothing strength $\alpha$ controls the boundary-error / lag-error tradeoff introduced in Eq.~\ref{eq:risk-decomposition}. Table~\ref{tab:alpha-sensitivity} sweeps $\alpha$ from 0 (single-step) to 0.95 on Llama-3.1-8B at 90\% compression.

\begin{table}[t]
\centering
\scriptsize
\setlength{\tabcolsep}{3pt}
\begin{tabular*}{\columnwidth}{@{\extracolsep{\fill}}lrrrrrrr}
\toprule
Benchmark & 0.00 & 0.30 & 0.60 & 0.80 & 0.90 & 0.95 & Best \\
\midrule
LongBench & 42.11 & 43.29 & 43.87 & 43.88 & \textbf{44.05} & 43.99 & 0.90 \\
LongBench-v2 & 0.296 & \textbf{0.298} & 0.290 & 0.294 & 0.294 & 0.288 & 0.30 \\
RULER & 53.97 & 63.55 & 67.03 & 75.60 & 78.55 & \textbf{78.72} & 0.95 \\
\bottomrule
\end{tabular*}
\caption{\textbf{Alpha sensitivity of \base at 90\% compression on Llama-3.1-8B.} Stronger smoothing improves LongBench and RULER; LongBench-v2 peaks at smaller $\alpha$.}
\label{tab:alpha-sensitivity}
\end{table}

The pattern directly reflects the two-term risk: retrieval-heavy benchmarks prefer strong temporal memory ($\alpha{=}0.90$--$0.95$), LongBench is robust around $0.80$--$0.90$, and LongBench-v2 peaks at $\alpha{=}0.30$, consistent with an adaptation-sensitive minority that penalizes excessive smoothing. We use $\alpha{=}0.80$ as a robustness compromise rather than a per-task optimum. Comparing EMA against cumulative (no-forgetting) aggregation is consistent with this picture: cumulative attention gains on Llama RULER (+2.01) but drops on Qwen RULER ($-$7.66), making EMA the more robust bounded-memory compromise (Appendix~\ref{sec:appendix-temporal}).

\paragraph{Why online adaptation did not help.}
We also tested a theory-derived adaptive selector that estimates noise and drift during an initial warmup. It almost always selects $\alpha{=}0.8$ or $0.95$ and cannot recover the low-$\alpha$ preference of LongBench-v2. During warmup, the noise-to-drift ratio is large ($\eta{=}50$--$600$), so the selector favors strong smoothing; useful relevance shifts arrive later. A reactive rank-disagreement trigger is also insufficient: disagreement correlates only 0.14 with alpha sensitivity and does not identify whether $\alpha$ should increase or decrease. Appendix~\ref{sec:adaptive-alpha} gives the full negative result. This suggests that successful adaptation requires directional detection of relevance shifts, not only early noise/drift estimation.

\paragraph{Short-prefix check.}
Because Score-Free does not accumulate EMA during prefill, prefix length is not an EMA-convergence requirement. On the needle-in-a-haystack stress test with 1k, 2k, 4k, and 8k prefills, Score-Free obtains ROUGE-L-F 0.71 at every length, matching Full Refresh and Lazy4. The more relevant failure mode is a later shift in token relevance: on LongBench-v2, 97.6\% of examples are invariant across temporal rules, while the remaining 2.4\% are adaptation-sensitive.

\subsection{Efficiency Analysis}
\label{sec:efficiency}

The previous sections establish the quality behavior; we now examine efficiency. At 64k context, 90\% compression reduces the logical final decode-side KV footprint from 64,029 to 6,401 retained tokens, an approximately $10\times$ reduction in retained-token count. This logical reduction does not by itself establish a $10\times$ serving-concurrency gain, which additionally depends on allocator behavior, page reclamation, prefill peaks, and batching. \base does not reduce prefill peak memory (29.91\,GB, same as full cache), unlike TOVA (22.88\,GB) which evicts during prefill, but achieves a similar final retained-token count (6,401 vs.\ 6,428). In the evaluated masking-based mode, KVzip retains 64,029 decode-side tokens and incurs $6.5\times$ prefill cost (31.54\,s vs.\ 4.86\,s). \method further reduces decode cost via lazy refresh, scoring only every fourth step and cutting layer-scoring operations by $4\times$ (Figure~\ref{fig:lazy-tradeoff} in Appendix~\ref{sec:appendix-details}).

\section{Conclusion}
Under the tested aggressive-compression settings, EMA makes approximately order-preserving scorer modifications largely indistinguishable at the eviction-set level, but this conclusion does not extend to scorers that substantially alter utility ranking. This distinction motivates \lazyfour, which offers 1.34--1.46$\times$ decode throughput relative to full-refresh InertiaKV with little quality change. Score-Free is a separate empirical endpoint: scoring once at the first decode step and freezing the ranking yields an average $\Delta{=}+0.03$, with one significant small regression on MultiNews. The results redirect attention toward temporal aggregation without claiming that scoring quality is generally irrelevant. Adaptive $\alpha$, relevance-shift detection, and competitive non-attention scorers remain open directions.

\section*{Limitations}
\base and \method are decode-only compressors that target the decode/time-between-tokens bottleneck; they do not reduce prefill peak memory and should be composed with a prefill-time compressor when prefill memory is binding. Score-Free uses only the first decode-step ranking and therefore may fail when relevance shifts later in generation; its stable 1k--8k needle-in-a-haystack result does not establish robustness for multi-turn dialogue or short-prompt, long-generation workloads. At 90\% compression, Mistral-7B and Qwen2.5-14B degrade by 17--27 RULER points mainly on multi-needle and distant-token retrieval. Because \lazyfour closely tracks full refresh on these models, the observed limitation lies in the underlying scorer/task interaction rather than lazy temporal aggregation. Because the implementation updates EMA per eligible layer, $\alpha$ couples layer weighting and temporal retention, limiting purely temporal interpretation and potentially architecture transfer. EMA is also not a universal wrapper: applying it to other scoring rules produces only marginal changes (Appendix~\ref{sec:appendix-transfer}). Finally, our evaluation is English-only and does not rule out a competitive non-attention scorer whose ranking differs from attention while retaining task utility.

\section*{Acknowledgments}
We thank the anonymous reviewers and the area chair for their valuable feedback and constructive suggestions. This research was jointly supported by Alibaba Group and the Excellence Funds of Universität Hamburg.

\bibliography{custom}

\begin{thebibliography}{27}
\providecommand{\natexlab}[1]{#1}

\bibitem[{Ainslie et~al.(2023)Ainslie, Lee-Thorp, de~Jong, Zemlyanskiy, Lebron,
  and Sanghai}]{ainslie2023gqa}
Joshua Ainslie, James Lee-Thorp, Michiel de~Jong, Yury Zemlyanskiy, Federico
  Lebron, and Sumit Sanghai. 2023.
\newblock \href {https://doi.org/10.18653/v1/2023.emnlp-main.298} {{GQA}:
  Training generalized multi-query transformer models from multi-head
  checkpoints}.
\newblock In \emph{Proceedings of the 2023 Conference on Empirical Methods in
  Natural Language Processing}, pages 4895--4901, Singapore. Association for
  Computational Linguistics.

\bibitem[{Bai et~al.(2024)Bai, Lv, Zhang, Lyu, Tang, Huang, Du, Liu, Zeng, Hou,
  Dong, Tang, and Li}]{bai2023longbench}
Yushi Bai, Xin Lv, Jiajie Zhang, Hongchang Lyu, Jiankai Tang, Zhidian Huang,
  Zhengxiao Du, Xiao Liu, Aohan Zeng, Lei Hou, Yuxiao Dong, Jie Tang, and
  Juanzi Li. 2024.
\newblock \href {https://doi.org/10.18653/v1/2024.acl-long.172} {{LongBench}: A
  bilingual, multitask benchmark for long context understanding}.
\newblock In \emph{Proceedings of the 62nd Annual Meeting of the Association
  for Computational Linguistics (Volume 1: Long Papers)}, pages 3119--3137,
  Bangkok, Thailand. Association for Computational Linguistics.

\bibitem[{Bai et~al.(2025)Bai, Tu, Zhang, Peng, Wang, Lv, Cao, Xu, Hou, Dong,
  Tang, and Li}]{bai2024longbenchv2}
Yushi Bai, Shangqing Tu, Jiajie Zhang, Hao Peng, Xiaozhi Wang, Xin Lv, Shulin
  Cao, Jiazheng Xu, Lei Hou, Yuxiao Dong, Jie Tang, and Juanzi Li. 2025.
\newblock \href {https://doi.org/10.18653/v1/2025.acl-long.183} {{LongBench
  v2}: Towards deeper understanding and reasoning on realistic long-context
  multitasks}.
\newblock In \emph{Proceedings of the 63rd Annual Meeting of the Association
  for Computational Linguistics (Volume 1: Long Papers)}, pages 3639--3664,
  Vienna, Austria. Association for Computational Linguistics.

\bibitem[{Behnam et~al.(2025)Behnam, Fu, Zhao, Tsai, Yu, and
  Tumanov}]{behnam2025rocketkv}
Payman Behnam, Yaosheng Fu, Ritchie Zhao, Po-An Tsai, Zhiding Yu, and Alexey
  Tumanov. 2025.
\newblock \href {https://proceedings.mlr.press/v267/behnam25a.html}
  {{RocketKV}: Accelerating long-context {LLM} inference via two-stage {KV}
  cache compression}.
\newblock In \emph{Proceedings of the 42nd International Conference on Machine
  Learning}, volume 267 of \emph{Proceedings of Machine Learning Research},
  pages 3358--3392, Vancouver, British Columbia, Canada. PMLR.

\bibitem[{Brandon et~al.(2024)Brandon, Mishra, Nrusimha, Panda, and
  Ragan-Kelley}]{brandon2024cla}
William Brandon, Mayank Mishra, Aniruddha Nrusimha, Rameswar Panda, and
  Jonathan Ragan-Kelley. 2024.
\newblock \href
  {https://proceedings.neurips.cc/paper_files/paper/2024/hash/9e23d020c18e4c40d81c6a0fc7a46f68-Abstract-Conference.html}
  {Reducing transformer key-value cache size with cross-layer attention}.
\newblock In \emph{Advances in Neural Information Processing Systems 37}, pages
  86927--86945, Vancouver, British Columbia, Canada. Curran Associates, Inc.

\bibitem[{Cai et~al.(2025)Cai, Zhang, Gao, Liu, Li, Liu, Lu, Xiong, Dong, Hu,
  and Xiao}]{pyramidkv2024}
Zefan Cai, Yichi Zhang, Bofei Gao, Yuliang Liu, Yucheng Li, Tianyu Liu, Keming
  Lu, Wayne Xiong, Yue Dong, Junjie Hu, and Wen Xiao. 2025.
\newblock \href {https://openreview.net/forum?id=ayi7qezU87} {{PyramidKV}:
  Dynamic {KV} cache compression based on pyramidal information funneling}.
\newblock In \emph{Proceedings of the Second Conference on Language Modeling},
  Montreal, Quebec, Canada.

\bibitem[{Dao(2024)}]{dao2023flashattention2}
Tri Dao. 2024.
\newblock \href
  {https://proceedings.iclr.cc/paper_files/paper/2024/hash/98ed250b203d1ac6b24bbcf263e3d4a7-Abstract-Conference.html}
  {{FlashAttention-2}: Faster attention with better parallelism and work
  partitioning}.
\newblock In \emph{The Twelfth International Conference on Learning
  Representations}, pages 35549--35562, Vienna, Austria.

\bibitem[{Feng et~al.(2025)Feng, Lv, Cao, Xie, and Zhou}]{li2024adakv}
Yuan Feng, Junlin Lv, Yukun Cao, Xike Xie, and S.~Kevin Zhou. 2025.
\newblock \href
  {https://proceedings.neurips.cc/paper_files/paper/2025/hash/a40ff56daab9f4808b1e18350c8a11ce-Abstract-Conference.html}
  {{Ada-KV}: Optimizing {KV} cache eviction by adaptive budget allocation for
  efficient {LLM} inference}.
\newblock In \emph{Advances in Neural Information Processing Systems 38}, pages
  113152--113188, San Diego, California, USA; Mexico City, Mexico. Curran
  Associates, Inc.

\bibitem[{Ge et~al.(2024)Ge, Zhang, Liu, Zhang, Han, and Gao}]{ge2024model}
Suyu Ge, Yunan Zhang, Liyuan Liu, Minjia Zhang, Jiawei Han, and Jianfeng Gao.
  2024.
\newblock \href
  {https://proceedings.iclr.cc/paper_files/paper/2024/hash/639a9a172c044fbb64175b5fad42e9a5-Abstract-Conference.html}
  {Model tells you what to discard: Adaptive {KV} cache compression for
  {LLMs}}.
\newblock In \emph{The Twelfth International Conference on Learning
  Representations}, pages 22975--22988, Vienna, Austria.

\bibitem[{Hsieh et~al.(2024)Hsieh, Sun, Kriman, Acharya, Rekesh, Jia, and
  Ginsburg}]{hsieh2024ruler}
Cheng-Ping Hsieh, Simeng Sun, Samuel Kriman, Shantanu Acharya, Dima Rekesh, Fei
  Jia, and Boris Ginsburg. 2024.
\newblock \href {https://openreview.net/forum?id=kIoBbc76Sy} {{RULER}: What's
  the real context size of your long-context language models?}
\newblock In \emph{Proceedings of the First Conference on Language Modeling},
  Philadelphia, Pennsylvania, USA.

\bibitem[{Joshi et~al.(2025)Joshi, Brahma, Liu, and Barsoum}]{joshi2025tada}
Vinay Joshi, Pratik~Prabhanjan Brahma, Zicheng Liu, and Emad Barsoum. 2025.
\newblock \href {https://doi.org/10.18653/v1/2025.acl-industry.101} {{TaDA}:
  Training-free recipe for decoding with adaptive {KV} cache compression and
  mean-centering}.
\newblock In \emph{Proceedings of the 63rd Annual Meeting of the Association
  for Computational Linguistics (Volume 6: Industry Track)}, pages 1435--1443,
  Vienna, Austria. Association for Computational Linguistics.

\bibitem[{Kim et~al.(2025)Kim, Kim, Kwon, Lee, Yun, and Song}]{zhang2025kvzip}
Jang-Hyun Kim, Jinuk Kim, Sangwoo Kwon, Jae~W. Lee, Sangdoo Yun, and Hyun~Oh
  Song. 2025.
\newblock \href
  {https://proceedings.neurips.cc/paper_files/paper/2025/hash/f4eaa4b8f2d08edb3f0af990d56134ea-Abstract-Conference.html}
  {{KVzip}: Query-agnostic {KV} cache compression with context reconstruction}.
\newblock In \emph{Advances in Neural Information Processing Systems 38}, pages
  167563--167591, San Diego, California, USA; Mexico City, Mexico. Curran
  Associates, Inc.

\bibitem[{Lee et~al.(2024)Lee, Lee, Seo, and Sim}]{lee2024infinigen}
Wonbeom Lee, Jungi Lee, Junghwan Seo, and Jaewoong Sim. 2024.
\newblock \href {https://www.usenix.org/conference/osdi24/presentation/lee}
  {{InfiniGen}: Efficient generative inference of large language models with
  dynamic {KV} cache management}.
\newblock In \emph{18th USENIX Symposium on Operating Systems Design and
  Implementation}, pages 155--172, Santa Clara, California, USA. USENIX
  Association.

\bibitem[{Li et~al.(2025)Li, Fu, Sheng, Long, Yu, and Li}]{li2025faedkv}
Runchao Li, Yao Fu, Mu~Sheng, Xianxuan Long, Haotian Yu, and Pan Li. 2025.
\newblock \href {https://doi.org/10.18653/v1/2025.findings-emnlp.914}
  {{FAEDKV}: Infinite-window fourier transform for unbiased {KV} cache
  compression}.
\newblock In \emph{Findings of the Association for Computational Linguistics:
  EMNLP 2025}, pages 16856--16866, Suzhou, China. Association for Computational
  Linguistics.

\bibitem[{Li et~al.(2024)Li, Huang, Yang, Venkitesh, Locatelli, Ye, Cai, Lewis,
  and Chen}]{xia2024snapkv}
Yuhong Li, Yingbing Huang, Bowen Yang, Bharat Venkitesh, Acyr Locatelli,
  Hanchen Ye, Tianle Cai, Patrick Lewis, and Deming Chen. 2024.
\newblock \href
  {https://proceedings.neurips.cc/paper_files/paper/2024/hash/28ab418242603e0f7323e54185d19bde-Abstract-Conference.html}
  {{SnapKV}: {LLM} knows what you are looking for before generation}.
\newblock In \emph{Advances in Neural Information Processing Systems 37}, pages
  22947--22970, Vancouver, British Columbia, Canada. Curran Associates, Inc.

\bibitem[{Liu et~al.(2023)Liu, Desai, Liao, Wang, Xie, Xu, Kyrillidis, and
  Shrivastava}]{liu2023scissorhands}
Zichang Liu, Aditya Desai, Fangshuo Liao, Weitao Wang, Victor Xie, Zhaozhuo Xu,
  Anastasios Kyrillidis, and Anshumali Shrivastava. 2023.
\newblock \href
  {https://proceedings.neurips.cc/paper_files/paper/2023/hash/a452a7c6c463e4ae8fbdc614c6e983e6-Abstract-Conference.html}
  {Scissorhands: Exploiting the persistence of importance hypothesis for {LLM}
  {KV} cache compression at test time}.
\newblock In \emph{Advances in Neural Information Processing Systems 36}, pages
  52342--52364, New Orleans, Louisiana, USA. Curran Associates, Inc.

\bibitem[{Liu et~al.(2024)Liu, Yuan, Jin, Zhong, Xu, Braverman, Chen, and
  Hu}]{liu2024kivi}
Zirui Liu, Jiayi Yuan, Hongye Jin, Shaochen Zhong, Zhaozhuo Xu, Vladimir
  Braverman, Beidi Chen, and Xia Hu. 2024.
\newblock \href {https://proceedings.mlr.press/v235/liu24bi.html} {{KIVI}: A
  tuning-free asymmetric 2bit quantization for {KV} cache}.
\newblock In \emph{Proceedings of the 41st International Conference on Machine
  Learning}, volume 235 of \emph{Proceedings of Machine Learning Research},
  pages 32332--32344, Vienna, Austria. PMLR.

\bibitem[{{Llama Team, AI at Meta}(2024)}]{grattafiori2024llama3}
{Llama Team, AI at Meta}. 2024.
\newblock \href {https://arxiv.org/abs/2407.21783} {The {Llama} 3 herd of
  models}.
\newblock \emph{arXiv preprint arXiv:2407.21783}.

\bibitem[{Oren et~al.(2024)Oren, Hassid, Yarden, Adi, and
  Schwartz}]{xiao2024tova}
Matanel Oren, Michael Hassid, Nir Yarden, Yossi Adi, and Roy Schwartz. 2024.
\newblock \href {https://doi.org/10.18653/v1/2024.emnlp-main.1043}
  {Transformers are multi-state {RNNs}}.
\newblock In \emph{Proceedings of the 2024 Conference on Empirical Methods in
  Natural Language Processing}, pages 18724--18741, Miami, Florida, USA.
  Association for Computational Linguistics.

\bibitem[{Park et~al.(2025)Park, Jones, Morse, Goel, Lee, and
  Lott}]{liu2025keydiff}
Junyoung Park, Dalton Jones, Matthew Morse, Raghavv Goel, Mingu Lee, and
  Christopher Lott. 2025.
\newblock \href
  {https://proceedings.neurips.cc/paper_files/paper/2025/hash/0907335ecf28faf15be54485dbcbe70e-Abstract-Conference.html}
  {{KeyDiff}: Key similarity-based {KV} cache eviction for long-context {LLM}
  inference in resource-constrained environments}.
\newblock In \emph{Advances in Neural Information Processing Systems 38}, pages
  5983--6019, San Diego, California, USA; Mexico City, Mexico. Curran
  Associates, Inc.

\bibitem[{{Qwen Team}(2024)}]{yang2024qwen2}
{Qwen Team}. 2024.
\newblock \href {https://arxiv.org/abs/2412.15115} {{Qwen2.5} technical
  report}.
\newblock \emph{arXiv preprint arXiv:2412.15115}.

\bibitem[{Sun et~al.(2025)Sun, Chang, Bao, Zheng, Zheng, Liu, Dong, Chi, and
  Chen}]{sun2025shadowkv}
Hanshi Sun, Li-Wen Chang, Wenlei Bao, Size Zheng, Ningxin Zheng, Xin Liu, Harry
  Dong, Yuejie Chi, and Beidi Chen. 2025.
\newblock \href {https://proceedings.mlr.press/v267/sun25b.html} {{ShadowKV}:
  {KV} cache in shadows for high-throughput long-context {LLM} inference}.
\newblock In \emph{Proceedings of the 42nd International Conference on Machine
  Learning}, volume 267 of \emph{Proceedings of Machine Learning Research},
  pages 57355--57373, Vancouver, British Columbia, Canada. PMLR.

\bibitem[{Tang et~al.(2024)Tang, Zhao, Zhu, Xiao, Kasikci, and
  Han}]{tang2024quest}
Jiaming Tang, Yilong Zhao, Kan Zhu, Guangxuan Xiao, Baris Kasikci, and Song
  Han. 2024.
\newblock \href {https://proceedings.mlr.press/v235/tang24h.html} {{Quest}:
  Query-aware sparsity for efficient long-context {LLM} inference}.
\newblock In \emph{Proceedings of the 41st International Conference on Machine
  Learning}, volume 235 of \emph{Proceedings of Machine Learning Research},
  pages 47901--47911, Vienna, Austria. PMLR.

\bibitem[{Vaswani et~al.(2017)Vaswani, Shazeer, Parmar, Uszkoreit, Jones,
  Gomez, Kaiser, and Polosukhin}]{vaswani2017attention}
Ashish Vaswani, Noam Shazeer, Niki Parmar, Jakob Uszkoreit, Llion Jones,
  Aidan~N. Gomez, {\L}ukasz Kaiser, and Illia Polosukhin. 2017.
\newblock \href
  {https://proceedings.neurips.cc/paper_files/paper/2017/hash/3f5ee243547dee91fbd053c1c4a845aa-Abstract.html}
  {Attention is all you need}.
\newblock In \emph{Advances in Neural Information Processing Systems 30}, pages
  5998--6008, Long Beach, California, USA. Curran Associates, Inc.

\bibitem[{Xiao et~al.(2024)Xiao, Tian, Chen, Han, and
  Lewis}]{liu2023streamingllm}
Guangxuan Xiao, Yuandong Tian, Beidi Chen, Song Han, and Mike Lewis. 2024.
\newblock \href
  {https://proceedings.iclr.cc/paper_files/paper/2024/hash/5e5fd18f863cbe6d8ae392a93fd271c9-Abstract-Conference.html}
  {Efficient streaming language models with attention sinks}.
\newblock In \emph{The Twelfth International Conference on Learning
  Representations}, pages 21875--21895, Vienna, Austria.

\bibitem[{Yu and Chai(2025)}]{yu2025evolkv}
Bohan Yu and Yekun Chai. 2025.
\newblock \href {https://doi.org/10.18653/v1/2025.findings-emnlp.88} {{EvolKV}:
  Evolutionary {KV} cache compression for {LLM} inference}.
\newblock In \emph{Findings of the Association for Computational Linguistics:
  EMNLP 2025}, pages 1673--1689, Suzhou, China. Association for Computational
  Linguistics.

\bibitem[{Zhang et~al.(2023)Zhang, Sheng, Zhou, Chen, Zheng, Cai, Song, Tian,
  R{\'e}, Barrett, Wang, and Chen}]{h2o2023}
Zhenyu Zhang, Ying Sheng, Tianyi Zhou, Tianlong Chen, Lianmin Zheng, Ruisi Cai,
  Zhao Song, Yuandong Tian, Christopher R{\'e}, Clark Barrett, Zhangyang Wang,
  and Beidi Chen. 2023.
\newblock \href
  {https://proceedings.neurips.cc/paper_files/paper/2023/hash/6ceefa7b15572587b78ecfcebb2827f8-Abstract-Conference.html}
  {{H$_2$O}: Heavy-hitter oracle for efficient generative inference of large
  language models}.
\newblock In \emph{Advances in Neural Information Processing Systems 36}, pages
  48951--48971, New Orleans, Louisiana, USA. Curran Associates, Inc.

\end{thebibliography}

\appendix

\section{Proof of Proposition~\ref{prop:rank-stability}}
\label{sec:appendix-proofs}

\begin{proof}
\textbf{(a)} Expanding the EMA recurrence gives $\mu_T(i) = (1{-}\alpha)\sum_{\tau=1}^{T}\alpha^{T-\tau}s_\tau(i) + \alpha^T \mu_0(i)$. Taking expectations, $\mathbb{E}[\mu_T(i)] = (1-\alpha^T)\bar{s}(i) + \alpha^T\mu_0(i) \to \bar{s}(i)$, so the expected gap converges to $\delta_{ij}$.
\textbf{(b)} By temporal independence, $\mathrm{Var}[\mu_T(i)] \leq (1{-}\alpha)^2\sigma^2\sum_{k=0}^{T-1}\alpha^{2k} \leq \frac{1-\alpha}{1+\alpha}\sigma^2$; independence across tokens makes the gap variance at most twice this quantity.
\textbf{(c)} Let $D_T=\mu_T(i)-\mu_T(j)$. By part (a), $\mathbb{E}[D_T]\to\delta_{ij}>0$. Applying Chebyshev's inequality to $D_T-\mathbb{E}[D_T]$ and then taking $\limsup_{T\to\infty}$, together with part (b), yields the stated asymptotic bound.
The i.i.d.\ and cross-token independence assumptions are for tractability. Real attention sequences are temporally dependent, so the proposition is an explanatory bound rather than a direct certificate for the empirical traces.
\end{proof}

\section{Controlled Auxiliary Ablations}
\label{sec:appendix-controlled}
The main paper treats static reserve and first-eviction rescue mechanisms as auxiliary design checks rather than as part of the method claim. Table~\ref{tab:method-delta} summarizes the controlled comparison. The auxiliary variant uses the same temporal accumulation rule as \base, but additionally protects a static reserve and applies a first-large-eviction rescue rule. These mechanisms were intended to protect coarse positional coverage and avoid early catastrophic eviction, but under the aggressive budgets studied here they do not reliably compose with the bounded-memory utility estimator.

\begin{table}[t]
\centering
\scriptsize
\setlength{\tabcolsep}{4pt}
\begin{tabular}{lcc}
\toprule
Mechanism & Auxiliary & \base \\
\midrule
Query-aware decoding score & \checkmark & \checkmark \\
Temporal accumulation & \checkmark & \checkmark \\
Static reserve & \checkmark & -- \\
First-eviction rescue & \checkmark & -- \\
Default $\alpha$ & 0.8 & 0.8 \\
\bottomrule
\end{tabular}
\caption{Mechanism-level delta between the controlled auxiliary variant and the final \base scorer. The auxiliary variant tests whether static reserve and first-eviction rescue compose with EMA temporal utility aggregation; it is not treated as an external baseline.}
\label{tab:method-delta}
\end{table}

A compression sweep shows that the gap between \base and the auxiliary variant grows as the budget becomes tighter on LongBench and RULER (Table~\ref{tab:sweep}). LongBench-v2 remains mixed, which is consistent with the main paper's interpretation that this benchmark exposes the lag side of temporal aggregation rather than supporting a universal design law.

\begin{table}[t]
\centering
\tiny
\setlength{\tabcolsep}{2pt}
\begin{tabular*}{\columnwidth}{@{\extracolsep{\fill}}llrrrr}
\toprule
Benchmark & CR & Auxiliary & w/ Reserve & \base & \base--Aux. \\
\midrule
LongBench & 0.75 & 44.04 & 46.06 & 45.07 & +1.03 \\
LongBench & 0.80 & 43.61 & 44.80 & 44.92 & +1.30 \\
LongBench & 0.90 & 42.13 & 41.95 & 43.88 & +1.75 \\
LongBench & 0.95 & 40.13 & 42.84 & 42.82 & +2.68 \\
\midrule
LongBench-v2 & 0.75 & 0.292 & 0.294 & 0.288 & -0.004 \\
LongBench-v2 & 0.80 & 0.298 & 0.296 & 0.290 & -0.008 \\
LongBench-v2 & 0.90 & 0.290 & 0.294 & 0.294 & +0.004 \\
LongBench-v2 & 0.95 & 0.304 & 0.294 & 0.296 & -0.008 \\
\midrule
RULER & 0.75 & 67.97 & 84.05 & 84.16 & +16.19 \\
RULER & 0.80 & 65.09 & 82.00 & 82.65 & +17.56 \\
RULER & 0.90 & 53.73 & 73.99 & 75.60 & +21.87 \\
RULER & 0.95 & 40.86 & 61.93 & 63.66 & +22.80 \\
\bottomrule
\end{tabular*}
\caption{Compression sweep on Llama-3.1-8B. The no-reserve \base design becomes increasingly stronger than the controlled auxiliary variant as compression becomes more aggressive, especially on RULER.}
\label{tab:sweep}
\end{table}

\begin{figure}[t]
\centering
\includegraphics[width=\columnwidth]{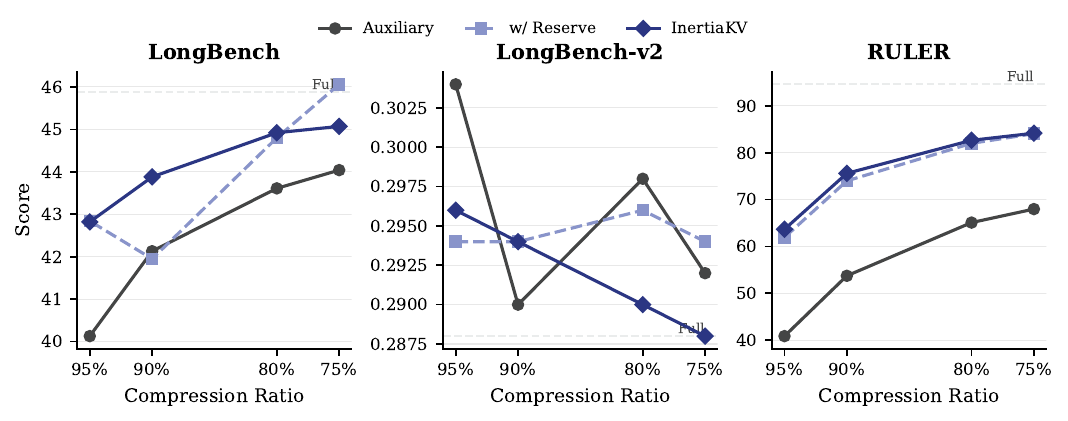}
\caption{Compression sweep visualization of Table~\ref{tab:sweep}. The gap between \base and the Auxiliary variant widens monotonically on LongBench and RULER as compression becomes more aggressive; LongBench-v2 remains mixed.}
\label{fig:compression-sweep}
\end{figure}

Table~\ref{tab:ablation} reports nearby operating-point variants at 90\% compression. Among these controlled variants, the no-reserve \base scorer is the strongest choice on LongBench and RULER. Removing momentum is especially harmful on RULER, while LongBench-v2 remains close across variants.

\begin{table}[t]
\centering
\scriptsize
\setlength{\tabcolsep}{3pt}
\begin{tabular*}{\columnwidth}{@{\extracolsep{\fill}}lrrrrr}
\toprule
Benchmark & Full & Reserve & No Mom. & \base & High Res. \\
\midrule
LongBench & 45.87 & 41.95 & 42.22 & \textbf{43.88} & 43.67 \\
LongBench-v2 & 0.288 & 0.294 & 0.296 & 0.294 & \textbf{0.298} \\
RULER & 94.59 & 73.99 & 53.61 & \textbf{75.60} & 71.18 \\
\bottomrule
\end{tabular*}
\caption{Auxiliary variants at 90\% compression on Llama-3.1-8B. These results motivate the base operating point but are not the main contribution.}
\label{tab:ablation}
\end{table}

\section{Auxiliary Generalization Checks}
\label{sec:appendix-auxiliary}
Having established the core method configuration in Appendix~\ref{sec:appendix-controlled}, we now test generalization beyond the main LongBench/RULER evaluation.

We use a needle-in-a-haystack retrieval grid as an auxiliary stress test even though the main claims rely on averaged LongBench, LongBench-v2, RULER, and profiling results. ROUGE-L-F is the most useful metric in this setup; exact match and hit rate are not informative. Figure~\ref{fig:needle} reports a 12-context-length (10k--120k) by 9-depth heatmap on Llama-3.1-8B. Full Cache is included as a reference, while compressed methods use the same 90\% setting. We use this grid as a fixed-setting stress check rather than as a benchmark for selecting the best alpha or refresh interval.

\begin{figure*}[t]
\centering
\includegraphics[width=\textwidth]{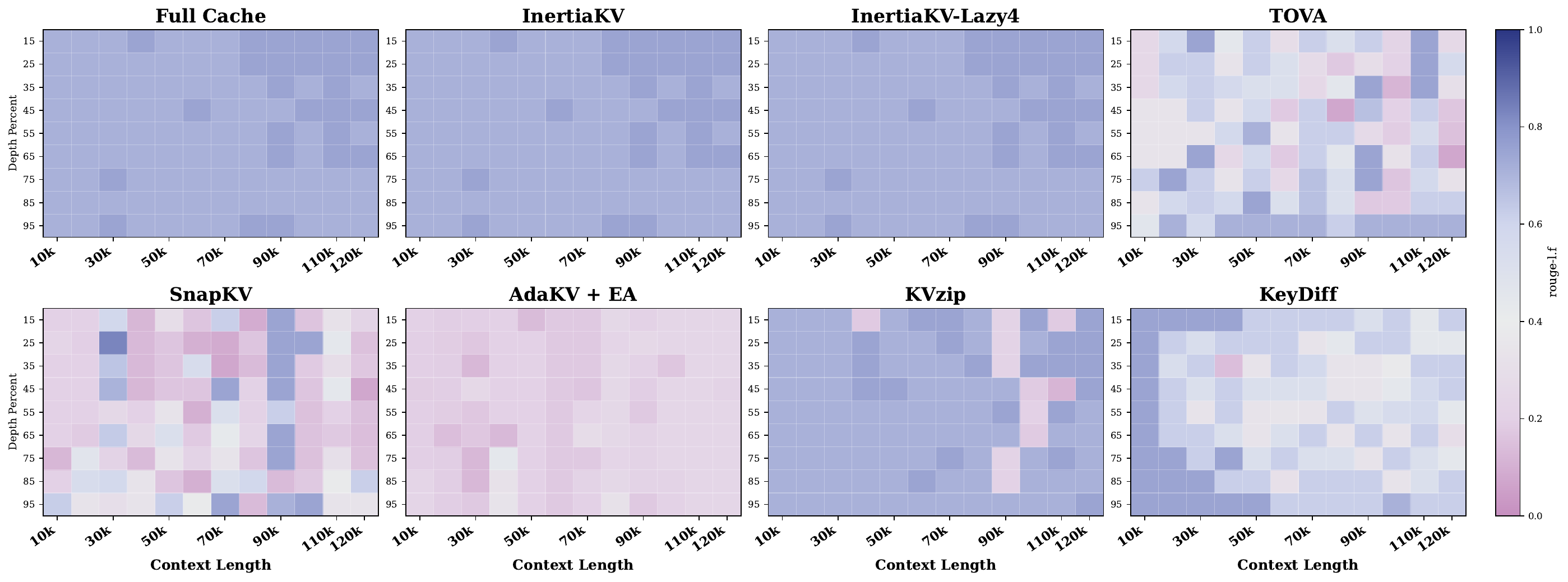}
\caption{Auxiliary needle-in-a-haystack heatmap on Llama-3.1-8B using ROUGE-L-F (12 context lengths $\times$ 9 depths). Full Cache is shown as a reference; compressed methods use 90\% compression. \base and \lazyfour closely track the full-cache pattern, while comparison baselines (TOVA, SnapKV, AdaKV+EA, KVzip, KeyDiff) degrade under this aggressive retrieval stress test.}
\label{fig:needle}
\end{figure*}

For a small non-dense stress check, we additionally evaluate a Qwen3-30B-A3B MoE model (Table~\ref{tab:moe}). This is not a full MoE benchmark suite. The same auxiliary-mechanism issue appears in this setting, although the method remains below full-cache quality.

\begin{table}[t]
\centering
\scriptsize
\begin{tabular*}{\columnwidth}{@{\extracolsep{\fill}}lrrrr}
\toprule
Benchmark & Full & Auxiliary & TOVA & \base \\
\midrule
LongBench subset & 44.91 & 40.39 & 37.30 & \textbf{42.32} \\
RULER & 95.51 & 38.28 & 33.56 & \textbf{66.26} \\
\bottomrule
\end{tabular*}
\caption{MoE stress check on Qwen3-30B-A3B at 90\% compression. The controlled auxiliary variant again underperforms \base.}
\label{tab:moe}
\end{table}

Table~\ref{tab:stress-transfer} reports two additional stress tests. These are not part of the main claim. They ask whether the same 90\% compression operating point transfers to a different dense architecture and to a larger Qwen model. The answer is mixed and mostly negative: \lazyfour continues to track the full-refresh \base estimator, but the estimator itself can degrade substantially on RULER. Subtask inspection shows that this degradation is concentrated in multi-key, multi-query, and context-word retrieval cases, while many single-key and QA subtasks remain much closer to full cache. On Qwen2.5-14B, for example, the 4k RULER average falls from 95.92 to 58.80 because \texttt{niah\_multikey\_3}, \texttt{niah\_multikey\_2}, and \texttt{niah\_multiquery} drop from 100.0, 100.0, and 99.8 to 0.0, 4.8, and 10.45; at 16k, the corresponding drops are 99.8$\rightarrow$2.0, 99.4$\rightarrow$57.8, and 99.75$\rightarrow$56.45. \lazyfour differs from full refresh by at most 0.65 points on these failed subtask averages, indicating that lazy refresh is not the primary source of the degradation. H2O-style cumulative scoring partially recovers some multi-evidence cases, indicating that no-forgetting accumulation can better preserve persistent evidence, but it still fails on the hardest multikey cases. We therefore treat these results as evidence of a current scorer limitation rather than as positive architecture or scale generalization. We omit the default DecodingPress StreamingLLM runs from this table because that wrapper uses a fixed decode interval and target size rather than the same 90\% compression-ratio protocol, and in these runs it often behaved as a near no-op.

\begin{table}[t]
\centering
\scriptsize
\setlength{\tabcolsep}{3pt}
\begin{tabular*}{\columnwidth}{@{\extracolsep{\fill}}llrrrr}
\toprule
Model & Benchmark & Full & \base & Lazy4 & H2O-style \\
\midrule
Mistral-7B & LongBench & 43.04 & 40.71 & 40.77 & 40.69 \\
Mistral-7B & LongBench-v2$^\dagger$ & 0.282 & 0.286 & 0.282 & 0.288 \\
Mistral-7B & RULER & 89.38 & 72.76 & 72.85 & 75.10 \\
\midrule
Qwen2.5-14B & LongBench & 48.68 & 46.01 & 46.11 & 46.03 \\
Qwen2.5-14B & LongBench-v2 & 0.328 & 0.326 & 0.326 & 0.330 \\
Qwen2.5-14B & RULER & 94.50 & 67.70 & 67.72 & 74.86 \\
\bottomrule
\end{tabular*}
\caption{Architecture and scale stress tests at 90\% compression. \lazyfour remains close to full-refresh \base, but the underlying scorer is not robust on RULER multi-evidence retrieval for these backbones. $^\dagger$Mistral LongBench-v2 is evaluated with a 32k context cap because the uncapped run exceeded single-GPU prefill memory.}
\label{tab:stress-transfer}
\end{table}

\section{Additional Transfer Results}
\label{sec:appendix-transfer}
Appendix~\ref{sec:appendix-auxiliary} showed mixed transfer of the \base scorer across architectures, with substantial limitations on multi-evidence retrieval. A separate question is whether the EMA temporal wrapper itself transfers to other scoring rules.

Our fixed-budget transfer experiments test whether a generic EMA wrapper can improve other decoding-time rules beyond the \base design. These experiments use explicit target sizes rather than nominal compression ratios, since decoding-time wrappers are controlled by budget and compression interval rather than by a prefill-style top-$k$ ratio.

Table~\ref{tab:appendix-ema} shows the transfer results. The overall picture is weak: EMA produces only small changes for TOVA, KeyDiff, and KNorm under fixed budgets. This is not a contradiction of the main result; it reinforces our scope. Temporal smoothing is useful in the studied TOVA-style irreversible eviction setting, but it is not a universal wrapper that automatically improves every scoring rule.

\begin{table*}[t]
\centering
\scriptsize
\setlength{\tabcolsep}{3pt}
\begin{tabular*}{\textwidth}{@{\extracolsep{\fill}}llrrrrrrrr}
\toprule
Benchmark & Budget & TOVA & TOVA+EMA(0.6) & TOVA+EMA(0.8) & TOVA+EMA(0.9) & KeyDiff & KeyDiff+EMA & KNorm & KNorm+EMA \\
\midrule
LongBench subset & 1024 & 53.67 & 53.78 & 53.73 & 53.69 & 48.56 & 48.64 & 27.39 & 27.61 \\
RULER & 1024 & 80.01 & 79.98 & 79.95 & 79.96 & 68.60 & 68.58 & 22.73 & 22.62 \\
LongBench subset & 2048 & 53.95 & 54.02 & 53.95 & 54.05 & 53.06 & 53.05 & 42.22 & 42.47 \\
RULER & 2048 & 88.92 & 88.95 & 88.92 & 88.92 & 80.93 & 81.01 & 43.51 & 44.01 \\
\bottomrule
\end{tabular*}
\caption{Fixed-budget EMA transfer results. EMA smoothing yields only small changes across these decoding-time rules, so we treat transfer as weak supporting evidence rather than as a primary contribution.}
\label{tab:appendix-ema}
\end{table*}

\section{Implementation Details}
\label{sec:appendix-implementation}

This appendix specifies key implementation choices that are not fully explicit in the main text. \base and \method are decode-only compressors: they do not reduce prefill peak memory, and their primary effect is on the final decode-side cache size and decode-time score-update cost.

\paragraph{Per-step scoring.} At each decode step, we extract the attention weights from the last query token to all cached positions: $a_t \in \mathbb{R}^{H \times L_t}$, where $H$ is the number of attention heads and $L_t$ is the current cache length. The per-step score is obtained by mean-pooling across heads: $s_t(i) = \frac{1}{H}\sum_{h=1}^H a_t^{(h)}(i)$. For GQA architectures (Llama-3.1), each KV head's attention is shared across its corresponding query-head group before pooling. Optional VNorm weighting multiplies each head's attention by the mean-normalized $\ell_2$ norm of the value states before pooling; entropy weighting multiplies each layer's contribution by the normalized entropy of its attention distribution.

\paragraph{Momentum initialization ($m_0$).} The momentum state is initialized to zeros for all positions: $m_0(i) = 0$ for $i = 1, \ldots, L_0$. This means that the first decode step's score determines the initial ranking entirely, and the expected momentum begins moving toward its stationary mean from step 1.

\paragraph{New token insertion during non-refresh steps.} When a new token is generated during a non-refresh step (i.e., $t \notin \mathcal{R}_r$ in \method), its momentum score is initialized to the mean momentum of the $L_t-1$ existing tokens: $m_t(\mathrm{new})=\frac{1}{L_t-1}\sum_{i=1}^{L_t-1}m_{t-1}(i)$. This places the new token at an average ranking position, allowing it to be scored and ranked at the next refresh. The new token can still be evicted on the same step if the budget is exceeded.

\paragraph{Layer scoring.} Let $K$ denote the number of eligible layers. Each eligible layer immediately updates the shared state: $m_{t,\ell}=\alpha m_{t,\ell-1}+(1-\alpha)s_{t,\ell}$ for $\ell=1,\ldots,K$, with $m_{t,0}=m_{t-1,K}$. Thus the prior step has weight $\alpha^K$ and layer $\ell$ has weight $(1-\alpha)\alpha^{K-\ell}$; layers are not uniformly averaged. For Lazy refresh, the layer hooks are deactivated on non-refresh steps, so no attention is captured and the momentum state remains frozen.

\paragraph{Eviction on non-refresh steps.} Even when scores are not refreshed, eviction still occurs if the cache exceeds the budget (e.g., after a new token is added). The stale momentum $m_{t-1}$ is used for the eviction decision. This ensures the budget constraint is always satisfied regardless of the refresh schedule.

\section{Benchmark-Level Interpretation Guide}
\label{sec:appendix-benchmarks}

Each benchmark in our evaluation tests a different aspect of the temporal aggregation tradeoff. We provide brief guidance on how to read results from each.

\paragraph{LongBench-v2.}
LongBench-v2 is consistently more mixed than LongBench and RULER. Alpha sensitivity, pairwise flip analysis, and the failure-anatomy statistics indicate that lower smoothing is preferable on a small but meaningful subset of examples, especially on Qwen2.5-7B. Most examples (95--98\%) are invariant across temporal scoring choices. We treat LongBench-v2 as a boundary case that exposes the \emph{lag} side of the tradeoff rather than as evidence that EMA is broadly harmful.

\paragraph{Needle-in-a-haystack stress test.}
ROUGE-L-F is the only informative metric in our setup; exact match and hit rate are not discriminative. We use this grid as a qualitative retrieval stress check (Figure~\ref{fig:needle}) rather than as a core averaged benchmark.

\paragraph{MoE.}
The Qwen3-30B-A3B experiment (Table~\ref{tab:moe}) is a minimal stress check. Its role is to check whether the auxiliary-mechanism pattern (static reserve hurts under aggressive budgets) also appears beyond dense backbones.

\paragraph{Architecture and scale stress tests.}
The Mistral-7B and Qwen2.5-14B results (Table~\ref{tab:stress-transfer}) should be read as identifying a \emph{scorer limitation}, not a temporal aggregation failure. In both cases, \lazyfour remains faithful to the full-refresh estimator (within 0.1 points), showing that lazy refresh is not the primary degradation source. The degradation is concentrated in RULER multi-evidence retrieval and appears tied to the full-refresh attention-derived scorer/task interaction under the evaluated eviction framework. On Mistral, increasing $\alpha$ from 0.8 to 0.9 improves RULER from 72.76 to 76.79, partially mitigating but not solving the issue. These results motivate future work on architecture-aware scoring and multi-span budget allocation.

\section{Temporal Sensitivity Analysis}
\label{sec:appendix-temporal}

This appendix characterizes the stability--adaptivity tradeoff across the temporal aggregation family, addressing the central design question: how much temporal memory should a decode-time compressor keep?

\subsection{Alpha Sensitivity}

The main alpha sensitivity sweep is reported in \S\ref{sec:temporal-sensitivity} (Table~\ref{tab:alpha-sensitivity}). Figure~\ref{fig:alpha-tradeoff} visualizes the normalized tradeoff curve.

\begin{figure}[t]
\centering
\includegraphics[width=\columnwidth]{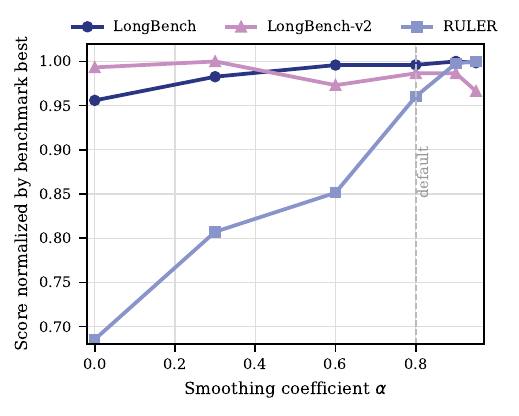}
\caption{Alpha sensitivity visualized. Each benchmark's score is normalized by its best value across the $\alpha$ range. RULER benefits strongly from high $\alpha$; LongBench is relatively flat above $\alpha{=}0.3$; LongBench-v2 slightly prefers lower smoothing. The dashed line marks the default $\alpha{=}0.8$.}
\label{fig:alpha-tradeoff}
\end{figure}

Here we provide cross-model validation. On Qwen2.5-7B, $\alpha=0.90$ is essentially tied with $\alpha=0.80$ on LongBench (46.05 vs.\ 46.01), but is worse on both LongBench-v2 (0.247 vs.\ 0.254) and RULER (78.28 vs.\ 79.43). We therefore keep $\alpha=0.80$ as the default: it is a selected operating point for the studied setting, not a tuned optimum.

LongBench-v2 provides a useful boundary case. Comparing $\alpha=0.30$ against higher smoothing, Qwen2.5-7B loses more examples than it gains: moving to $\alpha=0.80$ loses 13 previously correct examples and recovers 7, while moving to $\alpha=0.90$ loses 17 and recovers 7. Domain-level breakdowns are mixed, but Qwen shows broad declines in long structured data understanding, long-dialogue history understanding, multi-document QA, and single-document QA. This demonstrates the lag side of the tradeoff in the evaluated Qwen setting.

\subsection{Cumulative vs.\ EMA}

Because Eq.~\ref{eq:ema} is a bounded-memory filter, a natural alternative is the cumulative rule in Eq.~\ref{eq:cumulative}, which never forgets past evidence. Table~\ref{tab:h2o-style} tests whether the non-decaying endpoint dominates the bounded-memory endpoint once both are placed inside the same decode-time budgeted framework.

\begin{table}[t]
\centering
\scriptsize
\setlength{\tabcolsep}{4pt}
\begin{tabular}{llr}
\toprule
Model & Benchmark & Cum. - EMA \\
\midrule
Llama-3.1-8B & LongBench & +0.08 \\
Llama-3.1-8B & LongBench-v2 & -0.010 \\
Llama-3.1-8B & RULER & +2.01 \\
\midrule
Qwen2.5-7B & LongBench & -0.01 \\
Qwen2.5-7B & LongBench-v2 & -0.010 \\
Qwen2.5-7B & RULER & -7.66 \\
\bottomrule
\end{tabular}
\caption{Delta of H2O-style cumulative attention over EMA \base at 90\% compression. Cumulative attention is competitive but does not dominate.}
\label{tab:h2o-style}
\end{table}

The result is not one-sided. On Llama, cumulative attention is slightly better on LongBench and RULER but worse on LongBench-v2. On Qwen2.5-7B, it drops substantially on RULER ($-7.66$). Together with the $\alpha=0$ rows in Table~\ref{tab:alpha-sensitivity}, this completes the family picture: single-step scoring is too brittle for retrieval; cumulative accumulation can outperform EMA when evidence persists but degrades when stale retention matters. EMA is therefore not universally best, but it is the most robust bounded-memory compromise.

\section{Adaptive \texorpdfstring{$\alpha$}{alpha}: A Negative Result}
\label{sec:adaptive-alpha}

The alpha sensitivity in Table~\ref{tab:alpha-sensitivity} raises a natural question: can $\alpha$ be selected automatically? We tested a theory-derived adaptive selector that estimates noise ($\sigma$) and drift ($\delta$) during a warmup period, then selects $\alpha$ to minimize the boundary-error--lag tradeoff from Eq.~\ref{eq:risk-decomposition}.

The result is negative. Across both backbones, the selector almost always chooses $\alpha=0.8$ or $\alpha=0.95$, never selecting low $\alpha$. On RULER, this yields similar or slightly worse performance than the fixed $\alpha=0.8$ default; on LongBench-v2, where lower $\alpha$ helps, the selector cannot adapt. The failure mechanism is interpretable: cross-layer Jaccard disagreement ($\sigma \approx 0.07$) is stable, but temporal drift ($\delta \approx 0.0003$--$0.001$) is extremely small during the first 10 decode steps, producing noise-to-drift ratios $\eta = \sigma/\delta$ in the range 50--600 that the theory formula maps to high $\alpha$. Relevance shifts that would benefit from low $\alpha$ occur \emph{after} the warmup window.

This supports retaining $\alpha=0.8$ as a pragmatic default for the evaluated settings and identifies the failure mode a future adaptive policy must address: online detection of relevance shifts during generation, not just noise/drift estimation from initial steps. A reactive strategy---triggering lower $\alpha$ when raw/EMA ranking disagreement exceeds a threshold---is also unreliable: rank-trace diagnostics (Table~\ref{tab:rank-trace}) show that disagreement correlates with alpha-sensitivity (0.14--0.15) but does not reliably identify the correct direction.

\section{Mechanistic Diagnostics and Failure Anatomy}
\label{sec:appendix-diagnostics}

The preceding appendices characterize \emph{what} EMA does at the benchmark level; this appendix probes possible mechanisms by shifting from family-level averages to mechanism-level observables. We trace top-$k$ token sets during decoding on 100 examples per trace set, reporting raw top-$k$ churn, EMA top-$k$ churn, raw/EMA Jaccard overlap, the fraction of raw top-$k$ tokens evicted, and trace accuracy (Table~\ref{tab:rank-trace}).

We use churn as a proxy for the boundary-variance term in Eq.~\ref{eq:risk-decomposition}: high churn indicates score fluctuations can flip keep/drop decisions. Raw/EMA Jaccard and raw-top evicted rate proxy for lag: divergence means the aggregated estimate no longer tracks the current query-aware signal.

\begin{table*}[t]
\centering
\scriptsize
\begin{tabular*}{\textwidth}{@{\extracolsep{\fill}}llrrrrr}
\toprule
Trace Set & $\alpha$ & Raw Churn & EMA Churn & Raw/EMA Jaccard & Raw-Top Evicted & Trace Acc. \\
\midrule
LongBench-v2 & 0.30 & 0.0215 & 0.0005 & 0.9806 & 0.0109 & 0.29 \\
LongBench-v2 & 0.80 & 0.0375 & 0.0005 & 0.9659 & 0.0213 & 0.29 \\
LongBench-v2 & 0.90 & 0.0407 & 0.0005 & 0.9629 & 0.0238 & 0.28 \\
\midrule
RULER & 0.30 & 0.0178 & 0.0055 & 0.9868 & 0.0073 & 0.47 \\
RULER & 0.80 & 0.0244 & 0.0056 & 0.9799 & 0.0121 & 0.62 \\
RULER & 0.90 & 0.0284 & 0.0056 & 0.9762 & 0.0149 & 0.64 \\
\bottomrule
\end{tabular*}
\caption{Rank-trace diagnostics at 90\% compression on Llama-3.1-8B. EMA top-$k$ churn is consistently lower than raw churn. On LongBench-v2, stronger smoothing diverges from raw attention without improving accuracy; on RULER, it improves accuracy.}
\label{tab:rank-trace}
\end{table*}

EMA top-$k$ churn is far lower than raw churn, showing that temporal accumulation stabilizes rankings in these traces. On LongBench-v2, increasing $\alpha$ lowers raw/EMA agreement without improving traced accuracy (consistent with reduced adaptivity). On RULER, stronger smoothing improves traced accuracy (consistent with persistent retrieval evidence).

\paragraph{Failure anatomy.} Across the full 503-example LongBench-v2 split, 97.6\% of Llama-3.1-8B examples and 95.0\% of Qwen2.5-7B examples are invariant across temporal rules. The sensitive subset is small: on Qwen, 17/503 examples are adaptation-lag candidates (lower smoothing correct) versus 7/503 stability-help cases. On Llama, the counts are 3/503 and 1/503. LongBench-v2 thus exposes the lag side on a minority of examples rather than indicating that EMA is broadly harmful.

A larger 300-example Qwen pilot shows that raw/EMA ranking disagreement (correlation 0.14) and raw-top-$k$ eviction rate (0.15) are more informative observables for alpha-sensitivity than boundary margin (0.03) or score entropy (0.07). However, they identify sensitivity better than the correct \emph{direction}: both low-$\alpha$-wins and high-$\alpha$-wins appear in elevated-disagreement bins.

\section{Supplementary Result Details}
\label{sec:appendix-details}

This appendix collects detailed breakdowns and efficiency measurements omitted from the main text for space.

\paragraph{Learned scorer degradation by sequence length.}
The learned bilinear scorer's failure worsens with sequence length. On Llama-3.1-8B RULER, the combined deviation is $\Delta{=}{-}38.0$ at 4k and $\Delta{=}{-}52.2$ at 16k, with needle-in-a-haystack retrieval tasks collapsing from ${\sim}100$ to single digits (e.g., \texttt{niah\_multiquery} drops from 99.2 to 3.4 at 16k). This length dependence suggests that longer sequences amplify the cost of a misaligned utility ranking, since more tokens compete for limited budget slots.

\paragraph{Lazy refresh tradeoff visualization.}
Figure~\ref{fig:lazy-tradeoff} shows the wall-clock decode speedup of lazy refresh variants.

\begin{figure}[t]
\centering
\includegraphics[width=0.75\columnwidth]{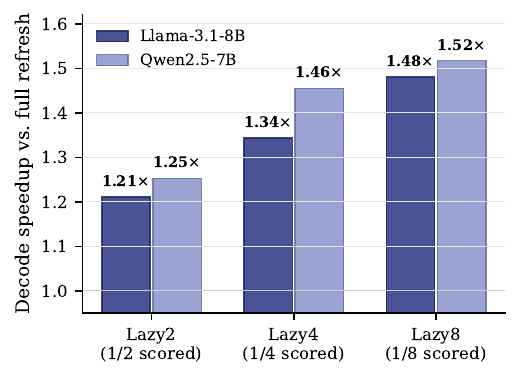}
\caption{Lazy refresh decode speedup vs.\ full-refresh InertiaKV at 90\% compression and 64k context. \lazyfour yields $1.34\times$ (Llama) and $1.46\times$ (Qwen) wall-clock speedup. For Lazy4, benchmark deltas are below 0.05 except for the small Llama RULER regression ($\Delta{=}{-}0.151$; Table~\ref{tab:lazy4-details}).}
\label{fig:lazy-tradeoff}
\end{figure}

\paragraph{Lazy4 per-benchmark quality deltas and confidence intervals.}
Table~\ref{tab:lazy4-details} reports the full per-benchmark breakdown for \lazyfour versus full refresh.

\begin{table}[ht]
\centering
\scriptsize
\setlength{\tabcolsep}{3pt}
\begin{tabular}{llrr}
\toprule
Model & Benchmark & $\Delta$ & 95\% CI \\
\midrule
Llama-3.1-8B & LongBench & +0.005 & $[-0.044, +0.053]$ \\
Llama-3.1-8B & LongBench-v2 & 0.000 & --- \\
Llama-3.1-8B & RULER & $-$0.151 & $[-0.223, -0.078]$ \\
\midrule
Qwen2.5-7B & LongBench & $-$0.026 & $[-0.111, +0.051]$ \\
Qwen2.5-7B & LongBench-v2 & $-$0.004 & --- \\
Qwen2.5-7B & RULER & +0.005 & $[-0.024, +0.033]$ \\
\bottomrule
\end{tabular}
\caption{\lazyfour quality delta versus full refresh at 90\% compression. LongBench-v2 has a single aggregate score, so no CI is computed. Both LongBench intervals and the Qwen RULER interval include zero; Llama RULER shows a small negative delta whose interval excludes zero.}
\label{tab:lazy4-details}
\end{table}

\begin{figure*}[t]
\centering
\includegraphics[width=\textwidth]{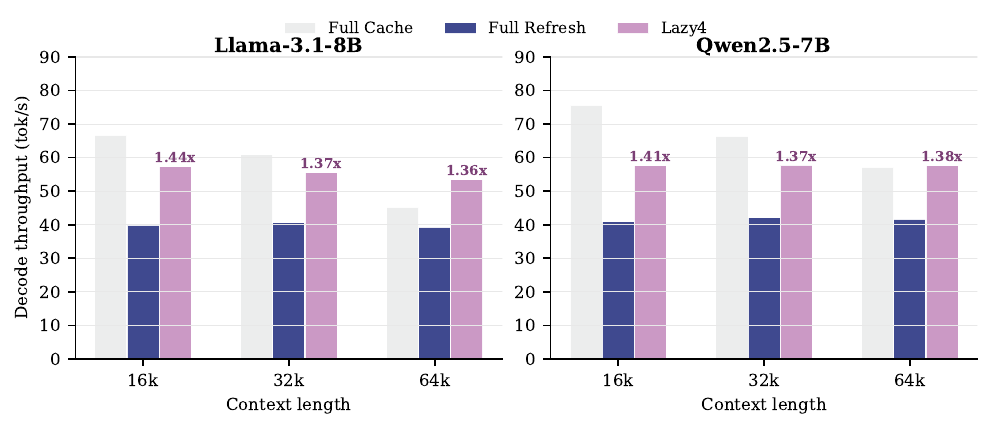}
\caption{Decode throughput scaling at 90\% compression. Each panel shows absolute throughput (tok/s) for Full Cache, Full Refresh \base, and \lazyfour at 16k, 32k, and 64k context. Annotations show \lazyfour speedup over Full Refresh, which remains stable at 1.36--1.44$\times$ across context lengths and models.}
\label{fig:efficiency-scaling}
\end{figure*}

\paragraph{Throughput scaling across context lengths.}
Figure~\ref{fig:efficiency-scaling} reports absolute decode throughput and Lazy4 speedup across 16k, 32k, and 64k contexts. Full Refresh throughput decreases with context length due to growing attention cost; Lazy4 maintains a 1.36--1.44$\times$ speedup at each tested length, indicating that the measured gain is not confined to one context length.

\paragraph{Hook-level profiling: scoring bottleneck.}
Hook-level profiling on Llama-3.1-8B at 64k context shows that the largest measured per-step cost is per-layer query-aware score reconstruction (specifically, attention fallback computation), not cache remapping or eviction logic. This motivates the lazy refresh family: it targets the largest measured cost component while leaving the retention budget and utility estimator unchanged.

\section{Reproducibility}
Reported result directories were archived and re-analyzed offline with duplicate metric identities treated as errors. For the central \lazyfour comparison, we also reran full-refresh \base and \lazyfour from fresh roots on both main backbones: eight LongBench tasks per backbone (16 model--task combinations), LongBench-v2 0-shot, and both reported RULER lengths. These clean-root reruns reproduce the task-level values in Figure~\ref{fig:lazy-tradeoff}, which we use as the canonical lazy-refresh quality result. Other experiments include architecture, scale, MoE, retrieval-grid, compression-sweep, and efficiency stress checks.

\end{document}